\documentclass{article}

\usepackage[main,final]{neurips_2026}

\usepackage[utf8]{inputenc} 
\usepackage[T1]{fontenc}    
\usepackage{hyperref}       
\usepackage{url}            
\usepackage{booktabs}       
\usepackage{amsfonts}       
\usepackage{nicefrac}       
\usepackage{microtype}      
\usepackage{xcolor}         

\title{Efficiently Representing Algorithms with Chain-of-Thought Transformers}

\author{
    Yanhong Li$^\dagger$ \quad Anej Svete$^\ddag$ \quad Ashish Sabharwal$^\dagger$ \quad William Merrill$^\dagger$ \\
    \\
    $^\dagger$ Allen Institute for AI, U.S.A. \quad $^\ddag$ ETH Z\"{u}rich, Switzerland \\
    \texttt{\{yanhong,ashishs,willm\}@allenai.org \quad asvete@ethz.ch}
}

\usepackage{times}
\usepackage{microtype}
\usepackage{amsmath,amssymb,amsthm,mathtools}
\usepackage{natbib}
\usepackage{hyperref}
\usepackage{bm}
\usepackage{bbm}
\usepackage{xcolor}
\usepackage{colortbl}
\usepackage{relsize}
\usepackage{paralist}

\hypersetup{
  hidelinks,
  colorlinks=true,
  citecolor=[RGB]{0,51,102},
  linkcolor=[RGB]{0,51,102},
  urlcolor=[RGB]{0,51,102}
}
\setcitestyle{round}

\definecolor{ETHBlue}{RGB}{33,92,175}	
\definecolor{ETHGreen}{RGB}{98,115,19}		
\definecolor{ETHPurpleDark}{RGB}{140,10,89}	
\definecolor{ETHPurple}{RGB}{163,7,116}	
\definecolor{ETHGray}{RGB}{111,111,111}	
\definecolor{ETHRed}{RGB}{183,53,45}	
\definecolor{ETHPetrol}{RGB}{0,120,148}	
\definecolor{ETHBronze}{RGB}{142,103,19}	

\colorlet{ETHdarkblue}{ETHBlue!80!black}
\colorlet{ETHdarkgreen}{ETHGreen!80!black}
\colorlet{ETHpink}{ETHPurple}
\colorlet{ETHgray}{ETHGray}
\colorlet{ETHred}{ETHRed}
\colorlet{ETHgreenblue}{ETHPetrol}
\colorlet{ETHbrown}{ETHBronze}

\definecolor{TextBlack}{RGB}{51,51,51}
\definecolor{BackgroundWhite}{RGB}{255,255,255}

\definecolor{AccentBlue}{RGB}{0,122,204}
\definecolor{LightBlue}{RGB}{173,216,230}
\definecolor{DarkBlue}{RGB}{0,51,102}

\definecolor{AccentGreen}{RGB}{70,160,73}
\definecolor{LightGreen}{RGB}{144,238,144}
\definecolor{DarkGreen}{RGB}{0,100,0}

\definecolor{AccentRed}{RGB}{255,0,0}
\definecolor{LightRed}{RGB}{255,99,71}
\definecolor{DarkRed}{RGB}{139,0,0}

\definecolor{AccentOrange}{RGB}{255,165,0}
\definecolor{LightOrange}{RGB}{255,204,153}
\definecolor{DarkOrange}{RGB}{255,140,0}

\definecolor{NeutralLightGray}{RGB}{204,204,204}
\definecolor{NeutralMediumGray}{RGB}{102,102,102}

\definecolor{NoteYellow}{RGB}{255,255,0}

\definecolor{DiversePurple}{RGB}{128,0,128}
\definecolor{DiverseTeal}{RGB}{0,128,128}
\definecolor{DiverseOlive}{RGB}{128,128,0}
\definecolor{DiverseCyan}{RGB}{0,128,192}
\definecolor{DiverseMagenta}{RGB}{192,0,128}

\definecolor{BrickRed}{rgb}{0.8, 0.1, 0.1}

\definecolor{MacroColor}{rgb}{0.0, 0.4, 0.8}
\newcommand{\mymacro}[1]{{#1}}
\newcommand{\opname}[1]{{\mymacro{\operatorname{#1}}}}
\usepackage{booktabs}
\usepackage[inline]{enumitem}
\usepackage{tikz}
\usetikzlibrary{arrows.meta,positioning,fit,decorations.pathreplacing,calligraphy,calc,backgrounds}
\usepackage{caption}
\usepackage{xspace}
\setlist[itemize]{leftmargin=14pt,topsep=2pt,itemsep=1.5pt}
\setlist[enumerate]{leftmargin=14pt,topsep=2pt,itemsep=1.5pt}

\usepackage{thmtools}
\usepackage{thm-restate}

\declaretheorem[name=Theorem]{theorem}
\declaretheorem[name=Lemma]{lemma}
\declaretheorem[name=Corollary]{corollary}
\declaretheorem[name=Proposition]{proposition}
\declaretheorem[name=Definition,style=definition]{definition}

\declaretheorem[name=Remark,style=remark]{remark}

\usepackage[capitalize,noabbrev]{cleveref}
\crefname{section}{\S}{\S\S}
\crefname{table}{Tab.}{Tabs.}
\crefname{figure}{Fig.}{Figs.}
\crefname{algorithm}{Alg.}{Algs.}
\crefname{equation}{Eq.}{Eqs.}
\crefname{example}{Example}{Examples}
\crefname{fact}{Fact}{Facts}
\crefname{appendix}{App.}{Apps.}
\crefname{theorem}{Thm.}{Thms.}
\crefname{aquestion}{Question}{Questions}
\crefname{assumption}{Assumption}{Assumptions}
\crefname{lemma}{Lem.}{Lemmata}
\crefname{conjecture}{Conj.}{Conjs.}
\crefname{proposition}{Prop.}{Props.}
\crefname{chapter}{Ch.}{Chapters}
\crefname{line}{line}{lines}
\crefname{principle}{Principle}{Principles}
\crefname{definition}{Def.}{Defs.}
\crefname{corollary}{Cor.}{Cors.}
\crefname{construction}{Construction}{Construction}
\crefname{takeaway}{Takeaway}{Takeaways}
\crefformat{section}{\S#2#1#3} 

\newcommand{\defn}[1]{\mymacro{\textbf{#1}}}
\newcommand{\defeq}{\mathrel{\stackrel{\textnormal{\tiny def}}{=}}}

\newcommand{\R}{\mymacro{\mathbb{R}}}

\newcommand{\N}{\mymacro{\mathbb{N}}}

\newcommand{\indFun}[1]{\mymacro{\mathbbm{1}[#1]}}
\newcommand{\set}[1]{\mymacro{\{#1\}}}

\newcommand{\norm}[1]{\left\lVert #1\right\rVert_2}
\newcommand{\BIT}{\opname{BIT}}
\newcommand{\bits}{\opname{bits}}
\newcommand{\bitsFun}[1]{\bits(#1)}

\newcommand{\inputstr}{{\mymacro{\boldsymbol{x}}}}        
\newcommand{\alphabet}{{\mymacro{\Sigma}}}

\newcommand{\inputsym}[1]{\mymacro{x_{#1}}}           

\newcommand{\EOS}{\mymacro{\$}}

\newcommand{\layernorm}{\opname{layernorm}}
\newcommand{\PLN}{\opname{PLN}}              
\newcommand{\LN}{\opname{LN}}               

\newcommand{\PC}{\opname{PC}}               

\newcommand{\PCstart}{\opname{PC}_{\mathrm{start}}}
\newcommand{\numregs}{{\mymacro{\rho}}}    
\newcommand{\score}{\opname{score}}        
\newcommand{\instrCost}{\mymacro{\kappa}}
\newcommand{\instrCostFun}[1]{\mymacro{\instrCost(#1)}}

\newcommand{\wRAM}{Word RAM\xspace}

\newcommand{\bigO}{{\mymacro{\mathcal{O}}}}
\newcommand{\bigOFun}[1]{{\mymacro{\bigO(#1)}}}
\newcommand{\bigOtilde}{{\mymacro{\tilde{\mathcal{O}}}}}

\newcommand{\bigOmega}{{\mymacro{\Omega}}}

\newcommand{\bigTheta}{{\mymacro{\Theta}}}
\newcommand{\bigThetaFun}[1]{{\mymacro{\bigTheta(#1)}}}

\newcommand{\wordsize}{{\mymacro{w}}}
\newcommand{\ramtime}{{\mymacro{t}}}
\newcommand{\cotstep}{{\mymacro{t}}}
\newcommand{\inputidx}{{\mymacro{i}}}
\newcommand{\inputlen}{{\mymacro{n}}}

\newcommand{\NTo}[1]{\mymacro{[#1]}}
\newcommand{\wordSet}{\mymacro{\NTo{2^{\wordsize} - 1}}}

\newcommand{\BOV}{\mymacro{\texttt{\#}}}

\newcommand{\classname}[1]{\mymacro{\relsize{-1} \textsf{#1}}\xspace}

\newcommand{\terminal}[1]{\mymacro{\texttt{#1}}}

\newcommand{\keytext}[1]{\terminal{\textcolor{ETHPurple}{#1}}}

\newcommand{\memkey}{\keytext{mem}\xspace}

\newcommand{\EOI}{\mymacro{\texttt{\&}}}

\newcommand{\uhat}{\classname{UHAT}}
\newcommand{\ahat}{\classname{AHAT}}

\newcommand{\posEnc}{\mymacro{\mathrm{PE}}}

\newcommand{\inner}[2]{\mymacro{#1^\top #2}}
\newcommand{\repr}{\opname{repr}}

\providecolor{ClaudeMini}{RGB}{150,30,160}

\usepackage[textsize=tiny]{todonotes}
\definecolor{mintgreen}{RGB}{152, 255, 152}
\makeatletter
\newcommand*\iftodonotes{\if@todonotes@disabled\expandafter\@secondoftwo\else\expandafter\@firstoftwo\fi}  %
\makeatother

\DeclarePairedDelimiter\ceil{\lceil}{\rceil}

\def\vzero{{{\mymacro{\bm{0}}}}}

\def\ve{{{\mymacro{\bm{e}}}}}

\def\vk{{{\mymacro{\bm{k}}}}}

\def\vq{{{\mymacro{\bm{q}}}}}

\def\vv{{{\mymacro{\bm{v}}}}}

\DeclareMathAlphabet{\mathsfit}{\encodingdefault}{\sfdefault}{m}{sl}
\SetMathAlphabet{\mathsfit}{bold}{\encodingdefault}{\sfdefault}{bx}{n}

\DeclarePairedDelimiter\abs{\lvert}{\rvert}
\DeclarePairedDelimiter\floor{\lfloor}{\rfloor}

\makeatletter
\let\oldnorm\norm
\def\norm{\@ifstar{\oldnorm}{\oldnorm*}}
\let\oldabs\abs
\def\abs{\@ifstar{\oldabs}{\oldabs*}}
\let\oldceil\ceil
\def\ceil{\@ifstar{\oldceil}{\oldceil*}}
\let\oldfloor\floor
\def\floor{\@ifstar{\oldfloor}{\oldfloor*}}
\makeatother

\usepackage{soul}
\usepackage{tcolorbox}

\newcommand{\dc}[3]{{\fboxsep=2pt\textsf{\small\color{#2!70!black}\sethlcolor{#2!10}\hl{\textbf{#1:} #3}}}}

\newcommand{\DC}[3]{%
    \begin{tcolorbox}[
        colback=#2!10, colframe=#2!10,
        fontupper=\small\sffamily\color{#2!70!black},
        boxsep=1pt, left=1pt, right=1pt, top=1pt, bottom=1pt,
        boxrule=0pt, arc=1pt
    ]
    \textbf{#1:} #3
    \end{tcolorbox}%
}

\newcommand{\registerdc}[3]{%
    \expandafter\newcommand\csname #2\endcsname[1]{\dc{#1}{#3}{##1}}%
    \expandafter\newcommand\csname #1\endcsname[1]{\DC{#1}{#3}{##1}}%
}

\registerdc{WM}{wm}{orange}
\registerdc{AS}{as}{purple}
\registerdc{YL}{yl}{cyan}
\registerdc{ASv}{asv}{lime!50!black}

\newcommand\upd[1]{{\color{blue} #1}}
\renewcommand\upd[1]{#1}

\begin{document}

\maketitle

\begin{abstract}
The increasing popularity of \emph{reasoning} models---language models that output a series of reasoning or thought tokens before producing an answer---is justified, in part, by theoretical results showing that chain-of-thought (CoT) transformers can simulate Turing machines, and thus perform arbitrary computation. 
However, the Turing machine, while suitable for complexity-theoretic analysis, is not convenient, intuitive, or efficient for discussing algorithms. 
Algorithms are typically designed and analyzed at a higher level of abstraction, captured by the \emph{Word RAM} model with random-access memory and unit-cost operations on $\bigO(\log n)$-bit words. 
As a result, Word RAM algorithms can be substantially more efficient than their Turing machine counterparts, raising the question: \emph{Can CoT transformers efficiently simulate Word RAM algorithms?} For instance, can they sort $n$ items in $\bigO(n \log n)$ steps or run Dijkstra's algorithm in $\bigO(E + V \log V)$ steps?
We answer affirmatively \upd{for three types of transformers}, up to poly-logarithmic overhead. 
We first establish this for fixed-precision transformers with poly-logarithmic width and rightmost unique hard attention, then strengthen the result to two more practical settings with fixed width and log-precision: \emph{continuous} CoT, where reasoning takes the form of vectors rather than tokens, and a \emph{hybrid} architecture in which transformer layers sit atop a recurrent (linear RNN) layer. 
In all three cases, we find that CoT \emph{can} efficiently simulate any Word RAM algorithm with only a poly-logarithmic overhead in $n$. 
This overhead reduces to log-square when the Word RAM has a ``flat'' instruction set, and only logarithmic for multiplication-free flat instructions---in stark contrast to known CoT simulations of Turing machines, which require quadratic overhead over Word RAM.
\end{abstract}

\section{Introduction}
It has been known for a while that chain of thought (CoT) makes transformers \emph{Turing complete}---able to simulate any algorithm in principle \citep{perez2021attention}. 
This justifies the empirical practice of training reasoning models to solve complex tasks \citep[e.g.,][]{snell2024scalingllmtesttimecompute,deepseek-r1,openai2026openaio1card}. 
However, the ability to simulate a Turing machine (TM) does not immediately translate to being able to implement interpretable algorithms efficiently: while TMs provide an established framework for analyzing computational complexity, they are not practical for designing real-world algorithms. 
For instance, sorting $\inputlen$ items takes only $\bigOFun{\inputlen \log \inputlen}$ steps on a random-access machine (RAM)---an idealized computer whose memory cells can be read and written in unit time at arbitrary addresses, formalized in \Cref{sec:wordram}---because each comparison and swap touches an arbitrary array entry in a single step.
A TM has no such random access: to reach the addressed cell it must move its head there, scanning a tape region of length up to $\bigO(\inputlen + \ramtime)$ per access.
Simulating these $\bigOFun{\inputlen \log \inputlen}$ RAM steps on a TM therefore costs up to $\bigO(\inputlen^2 \log \inputlen)$ steps, and on a single-tape TM sorting provably requires $\bigOmega(\inputlen^2 / \log \inputlen)$ steps \citep{HENNIE1965553}.
This overhead is not a quirk of the analysis---it reflects a fundamental mismatch between TMs and algorithms that rely on random memory access or word-level arithmetic. 
This paper asks: \emph{Can CoT transformers efficiently simulate algorithms as written in textbooks?} We answer affirmatively.

The natural model for describing textbook algorithms efficiently and concisely is the \defn{\wRAM{}}.
A \wRAM{} program stores data in a random-access memory (RAM) partitioned into \emph{words} of $\wordsize = \bigTheta(\log(\inputlen + \ramtime))$ bits---wide enough to address any of the $\bigO(\inputlen + \ramtime)$ cells the program may touch during a $\ramtime$-step execution---and it executes a program one instruction at a time, where each instruction runs in $\bigO(1)$ steps.
Its primitive operations---memory reads and writes, integer arithmetic, bitwise operations, and comparisons---correspond directly to the operations of modern programming languages and the algorithms textbooks describe.
Thus, the standard complexity bounds practitioners know about correspond naturally to the algorithms expressed in the \wRAM{} model.
For example, sorting takes $\bigO(\inputlen \log \inputlen)$ \wRAM{} steps and Dijkstra's algorithm takes $\bigO(E + V \log V)$.

Existing CoT expressivity results establish computational universality by simulating Turing machines \citep{perez2021attention,merrill2024cot}.
These constructions, however, operate entirely at the TM-instruction level: each CoT token encodes one tape symbol, and each decoding step simulates one TM transition.
Routing a \wRAM{} program of $\ramtime$ steps through this construction therefore costs $\bigO(\ramtime^2)$ CoT tokens: a \wRAM{} on a TM already incurs an $\bigO(\ramtime^2)$ blowup,\footnote{In general, each random-access step forces the TM to scan a tape region of length $\bigO(\ramtime)$ to reach the addressed cell.} and the CoT spends one token per TM step.\footnote{\upd{\citet{brosamle2026summarizedCoT} extend TM simulations to softmax transformers, but still require one CoT step per TM step. \citet{li2025constantbit-size,li2026efficient} extend prior results to constant bit-size transformers, but at the cost of multiple CoT steps per TM step. We aim for efficient simulations of Word RAM algorithms, which generally need substantially fewer steps than TMs.}}
What is missing is a direct simulation at the \wRAM{}-instruction level---one that would let CoT transformers execute textbook algorithms at close to their textbook complexity.

We close this gap with direct \wRAM{} simulations for decoder-only CoT transformers in three settings.
\begin{enumerate*}[label=\textit{(\roman*)}]

    \item First, \upd{we do this for} fixed-precision transformers of polylogarithmic width using rightmost-unique hard attention. This provides an interpretable and pedagogical construction emitting $\bigOFun{\wordsize^2}$ tokens per \wRAM{} instruction, hence $\bigO(\ramtime\, \wordsize^2) = \bigO(\ramtime \log^2(\inputlen + \ramtime))$ tokens overall, with the per-instruction overhead dropping to $\bigOFun{\wordsize}$ for multiplication-free instruction sets.
    This is essentially tight for discrete CoT: just as a Python program hard-coding a $p$-bit integer must write $\bigOmega(p)$ characters, any CoT transcript representing a $\wordsize$-bit word must emit $\bigOmega(\wordsize)$ tokens from a fixed alphabet, giving an $\bigOmega(\ramtime\, \wordsize) = \bigOmega(\ramtime \log(\inputlen + \ramtime))$ lower bound.
    Our \wRAM{} simulation matches this bound up to one logarithmic factor in general and exactly when multiplication is excluded---already far below the $\bigOmega(\ramtime^2)$ of TM-based CoT simulations.

    \item Second, \upd{we prove this for} \emph{fixed-width} transformers with \emph{continuous CoT}, which reasons with hidden vectors in place of discrete \upd{tokens~\citep{hao2025continuousCoT,gozeten2026continuousCoT}}.
    These models can invoke a TM oracle for each \wRAM{} instruction, generalizing the fixed-precision setting.

    \item Third, \upd{we show the same for} \emph{hybrid} models---transformers augmented with a linear \upd{recurrent (RNN)} layer~\citep{waleffe2024empiricalstudymamba,ren2025samba,kimi2025linear,nvidia2025nemotronh,merrill2026olmohybrid}---which can implement the same simulation with the usual, token-level CoT.

\end{enumerate*}
For an instruction set whose individual operations run in time $\instrCost(\wordsize)$ on a TM, the latter two constructions take $\bigO(\ramtime\, \instrCost(\wordsize))$ CoT steps, generalizing the results to arbitrary instruction sets.
This recovers the same poly-logarithmic overhead as the first setting: for the flat instruction sets we consider, $\instrCost(\wordsize) = \bigO(\wordsize)$ when multiplication, division, and modulo are excluded and $\bigO(\wordsize^2)$ otherwise, so $\bigO(\ramtime\, \instrCost(\wordsize))$ is $\bigO(\ramtime \log(\inputlen + \ramtime))$ or $\bigO(\ramtime \log^2(\inputlen + \ramtime))$ accordingly.

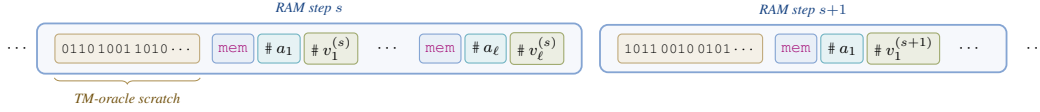
\begin{figure}[t]
    \centering
    \colorlet{KeyFill}{ETHBlue!10}
    \colorlet{KeyBorder}{ETHBlue!55}
    \colorlet{AddrFill}{ETHPetrol!10}
    \colorlet{AddrBorder}{ETHPetrol!55}
    \colorlet{ValFill}{ETHGreen!12}
    \colorlet{ValBorder}{ETHGreen!60}
    \colorlet{TMFill}{ETHBronze!10}
    \colorlet{TMBorder}{ETHBronze!55}
    \colorlet{SlotFill}{ETHBlue!4}
    \colorlet{SlotBorder}{ETHBlue!55}
    \resizebox{\linewidth}{!}{%
    \begin{tikzpicture}[
        font=\small,
        node distance=1mm,
        field/.style={draw, rounded corners=2pt, minimum height=5mm, align=center,
                      inner xsep=3pt, inner ysep=2pt, line width=0.55pt},
        key/.style ={field, fill=KeyFill,  draw=KeyBorder,  font=\footnotesize\ttfamily},
        addr/.style={field, fill=AddrFill, draw=AddrBorder, font=\footnotesize\ttfamily},
        val/.style ={field, fill=ValFill,  draw=ValBorder,  font=\footnotesize\ttfamily},
        tm/.style  ={field, fill=TMFill,  draw=TMBorder,  font=\scriptsize\ttfamily},
        delim/.style={font=\footnotesize\ttfamily, text=ETHGray!70!black},
        slot/.style={draw=SlotBorder, fill=SlotFill, line width=0.9pt,
                     rounded corners=4pt, inner xsep=3mm, inner ysep=1.5mm},
        slotbg/.style={slot, on background layer},
        brace/.style={decorate, decoration={brace, amplitude=3.5pt, raise=2pt}},
        ]

        \node[font=\footnotesize] (head) {$\cdots$};

        \node[tm, right=3mm of head, minimum width=22mm] (tmS)
            {0110\,1001\,1010\,$\cdots$};
        \node[key,  right=2mm of tmS]      (sM1k)   {\memkey};
        \node[addr, right=0.5mm of sM1k]   (sM1a)   {\texttt{\#}\,$a_1$};
        \node[val,  right=0.5mm of sM1a]   (sM1v)   {\texttt{\#}\,$v_1^{(s)}$};
        \node[draw=none, right=2mm of sM1v, font=\footnotesize] (sdots) {$\cdots$};
        \node[key,  right=2mm of sdots]    (sMlk)   {\memkey};
        \node[addr, right=0.5mm of sMlk]   (sMla)   {\texttt{\#}\,$a_\ell$};
        \node[val,  right=0.5mm of sMla]   (sMlv)   {\texttt{\#}\,$v_\ell^{(s)}$};
        \draw[brace, decoration={brace, amplitude=3pt, raise=7pt, mirror}, ETHBronze!75]
            (tmS.south west) -- (tmS.south east)
            node[midway, below=11pt, font=\scriptsize\itshape, text=ETHBronze!80!black]
            (tmSlbl) {TM-oracle scratch};
        \begin{scope}[on background layer]
        \node[fit=(tmS)(sMlv), slot,
              label={[font=\scriptsize\itshape, text=ETHBlue!75!black]above:RAM step $s$}]
              (slotS) {};
        \end{scope}

        \node[tm, right=6mm of slotS, minimum width=22mm] (tmT)
            {1011\,0010\,0101\,$\cdots$};
        \node[key,  right=2mm of tmT]      (tM1k)   {\memkey};
        \node[addr, right=0.5mm of tM1k]   (tM1a)   {\texttt{\#}\,$a_1$};
        \node[val,  right=0.5mm of tM1a]   (tM1v)   {\texttt{\#}\,$v_1^{(s+1)}$};
        \node[draw=none, right=2mm of tM1v, font=\footnotesize] (tdots) {$\cdots$};
        \begin{scope}[on background layer]
        \node[fit=(tmT)(tdots), slot,
              label={[font=\scriptsize\itshape, text=ETHBlue!75!black]above:RAM step $s{+}1$}]
              (slotT) {};
        \end{scope}

        \node[font=\footnotesize, right=2mm of slotT] (tail) {$\cdots$};

    \end{tikzpicture}}
    \caption{%
    \textbf{Schematic of the CoT transcript simulating a \wRAM{} program.}
    Within each slot, the model may emit \textcolor{ETHBronze!75!black}{TM-oracle scratch} tokens---intermediate computations used to carry out the instruction---before committing the resulting memory updates as a list of \textcolor{ETHPetrol!75!black}{address}--\textcolor{ETHGreen!70!black}{value} pairs.
    }
    \label{fig:block}
\end{figure}

\paragraph{Contributions.}
In summary, we prove that CoT transformers can directly simulate \wRAM{} algorithms with only poly-logarithmic overhead per instruction, matching an information-theoretic lower bound for discrete CoT up to a single logarithmic factor.
Specifically, as summarized in \Cref{tab:settings}:
\begin{enumerate}[label=(\roman*),leftmargin=14pt]
    \item \textbf{Polylogarithmic-width transformers (\Cref{sec:log-width}).} Fixed-precision transformers of $\log^2(\inputlen + \ramtime)$ width with rightmost-unique hard attention can simulate \wRAM{} programs. A \wRAM{} running $\ramtime$ steps with word size $\wordsize = \bigTheta(\log(\inputlen + \ramtime))$ is simulated using $\bigO(\ramtime \cdot \wordsize^2) = \bigO(\ramtime \log^2(\inputlen + \ramtime))$ tokens, dropping to $\bigO(\ramtime \cdot \wordsize) = \bigO(\ramtime \log(\inputlen + \ramtime))$ for multiplication-free instruction sets. The latter matches the $\bigOmega(\ramtime\, \wordsize)$ lower bound for discrete CoT exactly, the former up to a single $\wordsize$ factor.
    \item \textbf{Fixed-width transformers with continuous CoT (\Cref{sec:continuous-cot}).} A fixed-width, log-precision transformer augmented with continuous CoT---where each reasoning step is a real-valued vector rather than a discrete token---simulates any \wRAM{} program with $\bigO(\ramtime \cdot \instrCostFun{\wordsize})$ decoding steps, where $\instrCostFun{\wordsize}$ is the TM cost of executing one instruction. 
    \item \textbf{Hybrid transformer--RNN models (\Cref{sec:hybrid}).} A fixed-width transformer augmented with a single linear RNN layer simulates any \wRAM{} program at the same $\bigO(\ramtime \cdot \instrCostFun{\wordsize})$ cost with \emph{explicit} CoT, showing that recurrence is a natural and practical substitute for the width or continuity extensions of (i) and (ii).
\end{enumerate}

\begin{table}[t]
    \centering
    \small
    \caption{The three transformer settings in which we simulate \wRAM{}s and their trade-offs.
    Overheads are per simulated RAM step (CoT tokens per instruction); $\wordsize = \bigTheta(\log(\inputlen + \ramtime))$ is the word size and $\instrCostFun{\wordsize}$ is the TM cost of one instruction.
    The polylog-width overhead drops to $\bigO(\wordsize)$ for multiplication-free instruction sets; for the fixed-width settings, $\instrCostFun{\wordsize} = \bigO(\wordsize)$ (simple-arithmetic) or $\bigO(\wordsize^2)$ (full-arithmetic) flat instructions.}
    \label{tab:settings}
    \resizebox{\linewidth}{!}{%
    \begin{tabular}{@{}lcccccc@{}}
        \toprule
        Setting & Precision & Width & Attention & Uniform? & Overhead/step & Result \\
        \midrule
        Polylog-width (\Cref{sec:log-width})
            & fixed & $\bigO(\log^2 \inputlen)$ & rightmost \uhat{} & no
            & $\bigO(\wordsize^2)$ & \Cref{thm:logwidthSimulation} \\
        Continuous CoT (\Cref{sec:continuous-cot})
            & $\bigO(\log \inputlen)$ & fixed & averaging hard & yes
            & $\bigO(\instrCostFun{\wordsize})$ & \Cref{thm:wordram-ccot} \\
        Hybrid (\Cref{sec:hybrid})
            & $\bigO(\log \inputlen)$ & fixed & averaging hard {+} RNN & yes
            & $\bigO(\instrCostFun{\wordsize})$ & \Cref{thm:wordram-hybrid} \\
        \bottomrule
    \end{tabular}}
\end{table}

\section{Preliminaries}
\label{sec:prelim}

\subsection{The Word RAM Computational Model}
\label{sec:wordram}

There are several different variants of \wRAM{}s in the literature with different instruction set details, ranging from C-like \wRAM{}s \citep{hagerup1998wordram,thorup1996ram,morin2013ods} to the weaker multiplication-free basic instruction set of \citet{brodnik1997transdichotomous}.
While classical literature has often explored various \emph{minimal}, assembly-like instruction sets under which the model is Turing-complete, our goal here is different: we would like to argue that transformers and hybrid models, when equipped with CoT, can directly simulate programs even when they are written in a much higher-level language, such as ``textbook algorithms'' or simple Python programs, with only a log or log-square overhead. Note that this is a stronger claim than simulating a minimal instruction set with instructions like \terminal{load}, \terminal{store}, \terminal{increment}, etc., with a log or log-square overhead. Keeping this overhead small is important; after all, every \wRAM{} program can be converted to a Turing machine if one disregards the overhead (which, in general, would be quadratic), and then simulated using standard transformers with CoT via known constructions \citep{perez2021attention,merrill2024cot}.

\textbf{Syntax.}
To this end, we define the following grammar for specifying \wRAM programs. Let $\inputlen$ be the input length, $k$ the maximum number of variables (realized as registers) allowed in the \wRAM{} program, and \terminal{c1, $\ldots$, cm} the fixed constants (independent of $\inputlen$) allowed in the program.

\begin{small}
\begin{tabular}{rcl}
    \upd{Program} & \upd{$\rightarrow$} & \upd{Line | Program \terminal{;} Line} \\
    Line & $\rightarrow$ & Var \terminal{=} Expr | \terminal{halt} \\
    Expr & $\rightarrow$ & Input | Const | Var | Expr Op Expr | UnOp Expr | Expr \terminal{if} Cond \terminal{else} Expr \\
    Cond & $\rightarrow$ & Bool | Expr Comp Expr | Cond BoolOp Cond | \terminal{$\neg$} Cond \\
    Var & $\rightarrow$ & Reg | Mem | \terminal{pc} \\
    Mem & $\rightarrow$ & \terminal{mem}[Expr] \\
    Input & $\rightarrow$ & \terminal{inp}[Expr] \\
    Reg & $\rightarrow$ & \terminal{r1} | \terminal{r2} | $\ldots$ | \terminal{rk} \\
    Op & $\rightarrow$ & $+ | - | \times | \div | \bmod | \ll | \gg | \mathbin{\&} | \mathbin{|} | \oplus$ \\
    BoolOp & $\rightarrow$ & $\wedge | \vee$ \\
    UnOp & $\rightarrow$ & ${\neg}$ \\
    Comp & $\rightarrow$ & $< | \le | = | \neq | \ge | >$ \\
    Const & $\rightarrow$ & \terminal{\inputlen} | \terminal{c1} | \terminal{c2} | $\ldots$ | \terminal{cm} \\
\end{tabular}
\end{small}

The above grammar allows nesting of expressions (e.g., \terminal{mem}[\terminal{mem}[\terminal{mem}[7]]]) as well as composition of logical operators over conditionals via the Cond rule.
We require that the nesting depth of every expression and condition be bounded by a constant independent of $\inputlen$.
This is what makes each line executable in $\bigO(1)$ \wRAM{} steps.
We also consider a simplified version of this grammar that doesn't allow nesting. This is obtained by changing the grammar expansion rules for Expr and Cond to not be recursive, as follows:

\begin{small}
\begin{tabular}{rcl}
    FlatExpr & $\rightarrow$ & Input | Const | Var \\
    Expr & $\rightarrow$ & FlatExpr | FlatExpr Op FlatExpr | UnOp FlatExpr | FlatExpr \terminal{if} Cond \terminal{else} FlatExpr \\
    Cond & $\rightarrow$ & Bool | FlatExpr Comp FlatExpr
\end{tabular}
\end{small}

We call a \wRAM{} program \defn{flat} if it uses only the flat grammar above.
We further distinguish two flat instruction classes based on which operators from Op appear:
\defn{full-arithmetic} flat programs may use any operator from Op,
while \defn{simple-arithmetic} flat programs restrict Op to $\{+, -, \ll, \gg, \mathbin{\&}, \mathbin{|}, \oplus\}$, i.e., it excludes $\times$, $\div$, $\bmod$.

\textbf{Semantics.}
A specific \wRAM machine is specified by a fixed-length program of the form above, with the standard semantics when executed on a string $\inputstr \in \alphabet^*$.
In more detail, the machine has three memory arrays: a read-only one $\terminal{inp}$ that contains the input $\inputstr$, $\terminal{mem}$ is a standard read-write tape initialized to all zeros, and $\terminal{out}$ is an output tape that can only be written to, where written values are stored $\bmod \abs{\alphabet}$.
The machine state also specifies the value of the program counter $\terminal{pc}$ and fixed number of registers $\terminal{r1}, \ldots, \terminal{rk}$, which all start as 0.
When an instruction is run, $\terminal{pc}$ is incremented, unless it is otherwise assigned by the instruction (e.g., $\terminal{pc} = \terminal{r2} + 1$).
The program runs until $\terminal{halt}$ is executed, at which point the contents of $\terminal{out}$ are deemed the output.
When considering language recognition, we say that a \wRAM accepts a string if and only if $\terminal{out}[0] = 1$.

Without loss of generality, we can view the program counter $\terminal{pc}$ and each register $\terminal{ri}$ as stored in memory via
$\terminal{pc} \triangleq \terminal{mem}[0]$ and 
$\terminal{ri} \triangleq \terminal{mem}[-i]$, with $\terminal{mem}$ being 1-indexed for standard memory.

\Cref{app:increment-example} gives a worked example of simulating an increment instruction via CoT.

\subsection{Transformers and Related Models}
\label{subsec:transformers}

\textbf{Transformers.}
We defer a definition of standard transformers model to \Cref{app:transformer-model}.
The depth of our transformer will be \emph{fixed} (i.e., $\bigO(1)$) w.r.t.\ sequence length $\inputlen$ as well as the CoT runtime $t = \bigO(\inputlen^c)$.
Regarding precision and width of our transformers, we consider two possibilities motivated by a bifurcation in the transformer expressivity literature \citep{strobl-etal-2024-formal,svete2026revisiting}:

\begin{enumerate}
    \item Fixed precision and polylogarithmic (i.e., $\bigO(\log^k \inputlen)$) width. Because width grows with $\inputlen$ and $t$, this model is not fully uniform. \upd{It requires more parameters to be specified when operating on longer inputs or for computing positional encodings for} longer CoTs. We use rightmost-unique hard attention transformers (UHATs).
    
    \item Logarithmic (i.e., $\bigO(\log \inputlen)$) precision and fixed width. We further enforce in this model that the parameters are fully uniform\upd{, i.e., the parameters are fixed with respect to $\inputlen$ while the working precision grows with $\inputlen$}. We use averaging hard attention transformers (AHATs).
\end{enumerate}

\textbf{Linear RNNs and Hybrid Models.}
Next, we define linear RNN layers. There are many existing linear RNN architectures with complex parameterizations that all meet the following general form for different choices of $A_t$, $b_t$, and $c_t$ \citep{merrill2026linearrnnsparallelizable}:
    $S_{t+1} = A_t S_t + b_t$ and 
    $h_t = c_t^\top S_t$
We will assume for simplicity that $A_t$, $b_t$, and $c_t$ are all parameterized as projections of the input vector to the RNN layer at step $t$.
RNNs can be fit into a transformer architecture by using the head input to compute $A_t, b_t$, and $c_t$ and then defining $h_t$ as the head output.
The sublayer output from a multihead RNN layer would then be passed through layer-norm in the standard way.
A \emph{hybrid model} is an architecture that mixes attention heads with linear RNN heads.

\textbf{CoT.}
Finally, we say that a transformer $T$ computes a function $f$ within CoT budget $t(\inputlen)$ if, after feeding $\inputstr$ through $T$ and allowing at most $t(\abs{w})$ autoregressive steps sampled greedily (potentially terminating early), the generated tokens end with the suffix ${=}f(w) \EOS$, where ${=}$ is an output delimiter and $\EOS$ is the end-of-string symbol.
Continuous CoT is defined analogously except that final logits of the transformer are passed as a ``soft token'' to the next step rather than a token decoded from them.

\subsection{Layernorm Hash Representation}
\label{subsec:layernorm-hash}

In the log-precision model, we will represent values $v$ using their layer-norm hash $\phi(v)$, which is a unit-norm vector in $\R^4$ defined as follows~\citep{merrill2024cot}:
\begin{equation}
    \upd{\phi(u,v) \; := \; \layernorm(u,v,-u,-v) \; = \; \frac{(u, v, -u, -v)}{\sqrt{u^2 + v^2}} .}
\end{equation}
%
\upd{We will mostly use the special case $\phi(u) := \phi(u,1)$.}
This representation has several useful properties (similar to binary encodings) that we will leverage. 
The first property is useful for performing an equality check using inner product, as well as tie-breaking in order to simulate rightmost attention:

\begin{restatable}[Equality Check and Mismatch Gap: \citealp{merrill2024cot}]{proposition}{phiEqualityAndGapProp}\label{lem:phi1-gap}
    Let $u,v \in \{0,1,\dots,T\}$. Then $\langle \phi(u),\phi(v)\rangle=1$ if and only if $u=v$. Further, when $u \neq v$, $\phi(u),\phi(v)\rangle \le 1-\delta_T$ where $\delta_T := \frac{1}{4(T^2+1)^2}$.
\end{restatable}

See proof in \Cref{app:layernorm} for completeness. $\phi$ is scale invariant~\citep{merrill2024cot}: $\phi(cu, cv) = \phi(u,v) = \phi(u/v,1) = \phi(u/v)$. We can also ``swap'' the dimensions of $\phi$ to \emph{invert} a number: if $\phi(v) = (s_1, s_2, -s_1, -s_2)$, then $\phi(1/v) = (s_2, s_1, -s_2, -s_1)$.
We introduce a new general \emph{affine update} (proof in \Cref{app:layernorm}), computable by a fixed-depth feedforward network as well as with an RNN:

\begin{restatable}[Affine Update]{proposition}{affineUpdateProp}\label{prop:phi-affine}
    Fix constants $a,b \in \R$. Let $\phi(v) = (s_1, s_2, -s_1, -s_2)$. Then $\phi(av + b) = \layernorm(a s_1 + b s_2, s_2, -(a s_1 + b s_2), -s_2)$.
\end{restatable}

\section{Simulating \wRAM{}s with Polylog-Width Transformers}
\label{sec:log-width}

We begin by showing how transformers with polylogarithmic width and fixed precision can simulate \wRAM{}s---a stepping stone towards more \emph{uniform} transformer models, which do not require new parameters as the input length grows.
Thus, we use the construction in this section to cleanly outline the simulation; the main ideas will re-appear in the subsequent sections.

\textbf{The transformer model.}
We use a log$^2$-width fixed-precision transformer model; see \Cref{rem:log-squared-width} for a discussion of why log$^2$ width is needed, as opposed to log-width.
The constructions build on the log-width machinery formalized by existing work \citep{li2024chain,svete2026revisiting}.\footnote{In particular, we use \emph{$\mathsf{L}$-uniform} transformers \citep{london2025pausetokensstrictlyincrease}, \citep[][\S 4.4]{barcelo_et_al:DagRep.15.7.22} with positional encodings that only contain the binary representation of the position. This prevents the PEs from encoding too much information and relies on the transformer abilities to simulate the \wRAM programs.}
This allows us to reuse existing gadgets in our constructions. 

\textbf{Representation of values.}
In the polylogarithmic-width setting, word-sized integers are stored in the residual stream as binary vectors denoting their binary expansions: for a value $v \in \{0,1\}^\wordsize$, the \defn{representation} $\repr(v) \defeq \bits(v) \in \{0,1\}^\wordsize$ is placed directly as a vector in the residual stream.
This gives two primitives:
\begin{enumerate*}[label=\textit{(\roman*)}]
    \item \emph{equality checking} via inner products ($\repr(i)^\top \repr(j) = \wordsize$ iff $i = j$, and strictly less otherwise), and
    \item \emph{bit extraction} via $\inner{\repr(v)}{\ve_t} = b_t$ for the $t\textsuperscript{th}$ standard basis vector $\ve_t$.
\end{enumerate*}
These two primitives are the key gadgets exploited throughout this section; the subsequent sections will define different representations to maintain fixed width.

\Cref{thm:logwidthSimulation} formalizes \wRAM{} simulation by log$^2$-width transformers; see \Cref{app:logwidth} for the full proof.
\begin{restatable}[Simulation of \wRAM{}s by log$^2$-Width Transformers]{theorem}{logwidthSimulation} \label{thm:logwidthSimulation}
    Let $M$ be a flat \wRAM{} program with word size $\wordsize$.
    Then there exists a rightmost \uhat{}\; $T$ with fixed precision, $\bigOFun{\wordsize}$ heads each of width $\bigOFun{\wordsize}$, and total model dimension $\bigOFun{\wordsize^2} = \bigOFun{\log^2(\inputlen + t)}$, such that, for any $t \geq \inputlen$, $T$ can simulate $t$ steps of $M$ with $t \cdot \bigOFun{\wordsize^2}$ CoT steps.
    Further, if $M$ does not use $\times$, $\div$, or $\bmod$, $T$ can simulate $t$ steps of $M$ with $t \cdot \bigOFun{\wordsize}$ CoT steps.
\end{restatable}

\begin{proof}[Proof sketch]
    Following the \wRAM{} definition (\Cref{sec:wordram}), the \wRAM{} state---program counter, registers, and memory---is folded into a single random-access memory $\terminal{mem}$.
    The construction logs every memory write to the CoT transcript as a \emph{memory block} of the form $\memkey\,\BOV\,\bits(a)\,\BOV\,\bits(v)$ (\Cref{fig:block}).
    The central operation is the \emph{memory read} (\Cref{lem:reading-from-memory}): given a query address $\bits(a)$ in the residual stream, recover the value of the rightmost $\memkey$ block whose address field encodes $a$ (i.e., the most recent write to $\terminal{mem}[a]$).
    This uses two attention steps: the first selects the rightmost matching address position; the second retrieves the value field at the known offset.

    Each CoT step runs four constant-depth layer blocks (cf.\ \Cref{fig:logwidth-layer-blocks} and \Cref{app:logwidth}):
    \begin{enumerate*}[label=\textit{(\roman*)}]
        \item \textbf{Load \texttt{pc}:} an indexed read with the constant query $0$ recovers $\bits(\terminal{mem}[0])$ in two attention layers (\Cref{prop:logwidth-load-pc});
        \item \textbf{Load Instruction and Operands:} the active line of the program is identified from $\bits(\terminal{mem}[0])$ by an MLP that issues parallel reads of every compile-time address listed in the instruction's operand descriptor (registers), and a second round of reads resolves any runtime addresses ($\terminal{mem}[a]$ or $\terminal{inp}[a]$) whose index $a$ is now in the residual stream (\Cref{prop:logwidth-dispatch});
        \item \textbf{Execute:} a single MLP computes the $\wordsize$-bit destination value(s) and the next \texttt{pc} from the operands now in the residual stream (\Cref{prop:logwidth-execute});
        \item \textbf{Store:} an attention layer computes the within-step offset $\delta$ from the rightmost \EOI{}, and the appropriate $\delta^{\text{th}}$ token (a $\memkey$, a \BOV{}, an address bit, a value bit, or \EOI{}) is emitted as the output token (\Cref{prop:logwidth-store}).
    \end{enumerate*}

    For simple instructions, Blocks~\textit{(i)}--\textit{(iii)} produce identical outputs at every token within the same step, so only $\delta$ advances in Block~\textit{(iv)}, emitting one token at a time over $\bigOFun{\wordsize}$ total positions.
    Deserialization uses $\wordsize$ dedicated attention heads (one per bit), each of width $\wordsize$, giving the $\bigOFun{\wordsize^2} = \bigOFun{\log^2(\inputlen + t)}$ model dimension.
    In contrast to simple instructions, multiplication, division, and modulo \emph{cannot} be computed in a single transformer step and are compiled into $\bigOFun{\wordsize}$ flat micro-steps that share state through program-private memory cells.
    Each micro-step costs $\bigOFun{\wordsize}$ tokens, for $\bigOFun{\wordsize^2}$ total; every micro-step uses the same four-block pipeline (\Cref{prop:logwidth-mul-div}).
\end{proof}

\begin{corollary} \label{cor:logwidth-simulation}
    Let $M$ and $T$ be as in \Cref{thm:logwidthSimulation}.
    We have the following:
    \begin{enumerate}[label=(\roman*),leftmargin=20pt,noitemsep,topsep=0pt]
        \item $T$ can simulate $t$ steps of $M$ with \emph{full-arithmetic} flat instructions using $\bigO(t \log^2 t)$ steps.
        \item $T$ can simulate $t$ steps of $M$ with \emph{simple-arithmetic} flat instructions using $\bigO(t \log t)$ steps.
    \end{enumerate}
\end{corollary}

\begin{remark}[Bit extraction as a dot product]\label{rem:BIT-dot-product}
    Reading from and writing into the CoT transcript both require \emph{bit extraction}: computing $\BIT(v, t)$, the $t\textsuperscript{th}$ bit of $v$.
    In the log$^2$-width setting this is trivial: since $\bits(v) \in \{0,1\}^\wordsize$ is stored as a coordinate vector in the residual stream, extracting $b_t$ reduces to the dot product $\inner{\bits(v)}{\ve_t} = b_t$ with the $t\textsuperscript{th}$ standard basis vector $\ve_t$.
    In the fixed-width constructions later, $\bits(v)$ is no longer directly available as a coordinate vector, so bit extraction becomes substantially more involved and is one of the primary sources of complexity in those proofs.
\end{remark}

These results provide a general framework for how CoT can simulate \wRAM{} programs.
However, the log$^2$-width construction is \emph{non-uniform}: the model grows with $\inputlen$, so new parameters---untrained on shorter inputs---must be invoked as the inputs grow longer.
The construction also requires \emph{rightmost} \uhat{}; we show in \Cref{rem:rightmost-necessary} that weaker attention models cannot implement the indexed-read step.
The fixed-width constructions in the following sections avoid the non-uniformity and reliance on rightmost \uhat{}s at the cost of more complex reading and writing mechanisms.

\section{Towards a Uniform Fixed-Width Construction: Bottlenecks and Plan}\label{sec:fixed-width}


In \Cref{sec:log-width}, storing $\bits(v)$ in the residual stream gives three things for free:
\begin{enumerate*}[label=\textit{(\roman*)}]
    \item equality checking via inner products,
    \item bit extraction via $\inner{\bits(v)}{\ve_t} = b_t$, and
    \item a clean serialization/deserialization mechanism.
\end{enumerate*}
Moving to fixed width forces us to abandon the $\bits$ representation, and each of these primitives must be rebuilt from scratch.
We describe the resulting challenges below.

\textbf{Bottleneck: Reading and writing values.}
The deepest challenge in moving to fixed width is that we must adopt a different representation of numbers in the residual stream.
For fixed-width constructions, any such representation $\repr$ of values must satisfy two properties simultaneously:
\begin{enumerate}[label=(\alph*),leftmargin=20pt,noitemsep,topsep=0pt]
    \item \emph{Equality gap: } $\repr(i)^\top \repr(j) = 1$ if $i = j$, and $\repr(i)^\top \repr(j) < 1 - \epsilon$ otherwise.
    \item \emph{Read/write: } We can serialize $\repr(v)$ to a binary CoT encoding and recover $\repr(v)$ from it.
\end{enumerate}
The first property is needed for the indexed-read attention mechanism; the second is needed both 
\begin{enumerate*}[label=\textit{(\alph*)}]
    \item to pass pointers computed at a higher layer back to lower layers in future steps, and
    \item to hand binary-encoded operands to the TM oracle for complex arithmetic.
\end{enumerate*}
Past work has established the \defn{layer-norm hash} $\phi(i)$ as a fixed-width representation satisfying the equality-gap property \citep{merrill2024cot}, but reading and writing via $\phi$ presents a challenge.

Writing is hard because it requires computing $\BIT(v, j)$ from $\phi(v)$---mapping a hash representation back to individual bits.\footnote{$\BIT(v, j)$ plays a central role in descriptive complexity theory and the analysis of transformer expressivity.}
One approach is to compute $\phi(v \bmod 2)$ and $\phi(\lfloor v/2 \rfloor)$ iteratively.
When $v \leq t$, \citet{merrill2025little} give a construction for computing $\phi(v \bmod 2)$ and $\phi(\lfloor v/2 \rfloor)$ from $\phi(v)$; for larger values this is open (note that \Cref{prop:phi-affine} allows easily computing $\phi(v/2)$, but not $\phi(\lfloor v/2 \rfloor)$).
To resolve the write challenge, we adopt a normal form for \wRAM{}s that enforces $v \leq t$ at every step, developed in \Cref{sec:restricted-word-rams}, allowing us to compute the least significant bit of $v$ as well as $\phi(\lfloor v/2 \rfloor)$. The remaining bits of $v$ can be obtained by iterating the process, which requires repeatedly \emph{conditioning} on $\phi(\lfloor v/2 \rfloor)$, $\phi(\lfloor v/4 \rfloor)$, etc. We show this can be achieved via either of two minimal architectural extensions: continuous CoT (\Cref{sec:continuous-cot}, \Cref{lem:ccot-serialization}) or hybrid transformer--RNN models (\Cref{sec:hybrid}, \Cref{lem:hybrid-serialization}).\footnote{\label{footnote:log-depth}This iterative computation can also be performed by a logarithmic-depth transformer; we do not explore this here.}

Reading is hard for a complementary reason: while $\phi(v)$ can be reconstructed from a binary CoT encoding via iterative bit-shifting within the $\phi$ representation, this is not known to be achievable by a plain transformer.
We show that this too becomes
feasible with either of two minimal architectural extensions: continuous CoT (\Cref{sec:continuous-cot}, \Cref{lem:ccot-deserialization}) and hybrid transformer--RNN models (\Cref{sec:hybrid}, \Cref{lem:hybrid-deserialization}).\footref{footnote:log-depth}

\subsection{Writing Solution: Equivalent RAMs with Bounded Word Values} \label{sec:restricted-word-rams}

We define a \defn{$c$-dimensional} \wRAM as a \wRAM where the tape is replaced by a $c$-dimensional array of $c$-dimensional tuples.
The following lemma shows that such \wRAM{}s are as powerful as standard \wRAM{}s, up to a constant factor in runtime (see \Cref{app:cdim-wram} for a proof):
\begin{restatable}{lemma}{cDimwRAMLemma} \label{lem:cdim-wram}
    Let $M$ be a 1-dimensional \wRAM with word size $w = c \log(n + t)$ for some $c \in \mathbb N$.
    Then there exists a $c$-dimensional restricted \wRAM $M'$ with word size $w' = \log(n + t)$ and the same instruction set as $M$ that
    simulates $t$ steps of $M$ using $\bigO(t)$ steps.
\end{restatable}

This shows that we can reduce any $M$ using word size $w = c \log (n + t)$ to $M'$ using word size $w = \log (n + t)$.
Now, define a \wRAM as \defn{restricted} if, for any step $i$, any value $v$ represented in any register satisfies $v \leq i$.
Any \wRAM $M$ can be converted to be restricted
by making it do nothing for the first $v$ steps, and then simulating the $i$-th step of $M$ at step $v + i$:

\begin{lemma} \label{lem:restricted-wram}
    Let $M$ be any $c$-dimensional \wRAM where all values are at most $v(n)$ on inputs of size $n$.
    Then there exists $M'$ that simulates $t$ steps of $M$ using $v(n) + t$ steps such that, at each step $i$, all values in $M'$ are at most $i$.
\end{lemma}

We can first invoke \Cref{lem:cdim-wram} and then apply \Cref{lem:restricted-wram} with $v(n) = n + t$ to obtain the following:

\begin{theorem}[Restricted Word RAM Equivalence]\label{thm:wordram-equivalence}
    A standard \wRAM running in time $t(n)$ on inputs of size $n$ can be simulated by a restricted $c$-dimensional \wRAM in time $\bigO(n + t(n))$.
\end{theorem}

This will be useful because simulating restricted \wRAM{}s with transformer-like models will be easier than simulating unrestricted ones. Specifically, the property $v \leq i$ will be useful for writing numbers in binary as it allows us to divide numbers using attention via the following special case of the ``division lemma'' of \citet{merrill2025little} where the divisor is $2$ (see \Cref{app:cdim-wram} for a proof):
\begin{restatable}[Simplified Division Lemma]{lemma}{simplifiedDivisionLemma}\label{lem:division}
    There exists a transformer block that, for any $\phi(v)$ present in the residual stream at position $t$ for $v \leq t$, adds $\phi(\floor{v / 2})$ and $\phi(v \bmod 2)$ to the residual stream.
\end{restatable}

This will be useful for writing numbers $\phi(v)$ onto the CoT in binary because it allows computing the least significant bit via $\phi(v \bmod 2)$, or $\phi(v \bmod 2^i)$, which can be useful for recovering other bits.

\section{Simulating \wRAM{}s with Continuous CoT}\label{sec:continuous-cot}

The bounded-value normal form from \Cref{sec:restricted-word-rams} ensures a crucial \emph{bit-extraction} precondition:
during execution, every word $v$ at RAM step $s$ satisfies $v \le s$.
Since each simulated step produces at least one transcript symbol, this implies $v \le \tau$ whenever the simulator
begins emitting a bit-string segment at transcript position $\tau$.
Thus, whenever we need to stream a word as bits we may apply the simplified division primitive
(\Cref{lem:division}) to compute $\phi(\lfloor v/2\rfloor)$ and $\phi(v \bmod 2)$ (the first bit) from $\phi(v)$, and then repeat the process for the next bit by \emph{conditioning} the computation on $\phi(\lfloor v/2\rfloor)$, and so on.

The remaining obstacle is \emph{persistence}:
after streaming $w$ discrete bits, later computation must still have access to a constant-dimensional
\emph{word code} that supports exact-match attention (for RAM addressing) and can be carried forward across time.

We show that \emph{continuous CoT} can be used to achieve both conditioning and persistence.

\textbf{Continuous CoT and soft tokens.}
In the continuous-CoT model, each decoding step emits both
(i) a discrete token from a fixed alphabet and (ii) a constant-dimensional \emph{soft token}
$p_i \in \R^{d_s}$.
The next decoding step receives the entire prefix, including the past soft tokens, so $p_i$ can be attended to by future layers. In particular, a constant-depth transformer block can copy designated coordinates of the most recent soft token forward (e.g.\ via a fixed-offset retrieval head), so the soft token acts as a small persistent register bank threaded through the transcript.

\begin{theorem}[Simulating Word RAMs with CCoT]\label{thm:wordram-ccot}
    Let $M$ be a $d$-dimensional-array \wRAM with word size $w=\log(n+t)+\bigO(1)$ and per-instruction bit-serial cost bound $\instrCostFun{\wordsize}$.\footnote{Here $\instrCostFun{\wordsize}$ counts \emph{decoding steps} of a fixed-alphabet bit-serial implementation of a RAM instruction on $w$-bit operands, including the unavoidable $\Omega(w)$ serialization of any fresh $w$-bit output word over a fixed alphabet.}
    Then there exists a log-precision, fixed-width \emph{continuous CoT} transformer $T$ such that, for any $t\ge n$, $t$ steps of $M$ can be simulated by $T$ using $\bigO(t\,\instrCostFun{\wordsize})$ decoding steps.
\end{theorem}

\begin{proof}[Proof overview.]
    We prove \Cref{thm:wordram-ccot} by induction on the simulated RAM step $s$.
    The simulator maintains an \emph{append-only} transcript whose soft token stores the pair $(\phi(a),\phi(v))$ for the written address/value $\terminal{mem}[a]\leftarrow v$ (absorbing $\terminal{pc}$ and registers into $\terminal{mem}$; see \Cref{sec:wordram}).
    Checkpoint and store-summary positions are marked by constant discrete tags so attention can restrict to the relevant subset.
    If no store-summary exists for an address $a$, the value at $\terminal{mem}[a]$ is interpreted as $0$.
    A useful property of continuous CoT is that we can always make the latest checkpoint available ``from layer~0'' at future positions:
    a constant-depth carry block can attend to the rightmost checkpoint tag and copy its soft-token coordinates into designated
    coordinates of the current residual stream.
    
    \smallskip
    \noindent
    \textbf{One simulated RAM step.}
    Assume the invariant holds at the end of RAM step $s$.
    To simulate step $s{+}1$, we perform the following conceptual phases (some happen within a single decoding step, others take many decoding steps), illustrated in \Cref{fig-preview:ccot-phase-blocks}:
    \begin{enumerate}[label=(\arabic*),leftmargin=20pt,noitemsep,topsep=0pt]
        \item \textit{Load \terminal{pc}.}
        Attend to the rightmost checkpoint and copy $\phi(\PC^{(s)})$ (and the register codes) into the current residual stream.
        Use $\PC^{(s)}$ to select which instruction subroutine/oracle to run.
        
        \item \textit{Load operand values.}
        For each register/constant operand, read $\phi(\cdot)$ directly from the copied checkpoint state.
        For each memory operand $\terminal{mem}[a]$, form $\phi(a)$ and dereference the most recent matching store-summary
        to recover the operand in coded form $\phi(v)$ (Lemma~\ref{lem:ccot-dereferencing}).
    
        \item \textit{Serialize inputs (write a binary tape).}
        For each coded input word needed by the instruction oracle, emit its $w$ bits onto the discrete transcript using Lemma~\ref{lem:ccot-serialization}.
        This creates a delimited \emph{work-tape suffix} for the instruction oracle.
    
        \item \textit{Instruction execution (TM oracle).}
        Activate the instruction-specific TM-oracle subroutine, which operates only on the work-tape suffix and runs for
        $\bigO(\instrCostFun{\wordsize})$ decoding steps, emitting the output words as bit streams
        (Lemma~\ref{lem:ccot-instruction-execution}).
    
        \item \textit{Deserialize outputs.}
        Convert each output bit stream into its coded form $\phi(\cdot)$ for future addressing and equality tests,
        using Lemma~\ref{lem:ccot-deserialization}.
    
        \item \textit{Store updates (re-establish the invariant).}
        Emit a new checkpoint token whose soft token stores the updated state
        $\phi(\PC^{(s+1)}),\phi(r_1^{(s+1)}),\ldots,\phi(r_{\numregs}^{(s+1)})$.
        If the instruction writes $\terminal{mem}[a]\leftarrow v$, also emit a store-summary token storing $(\phi(a),\phi(v))$.
        By construction, this new checkpoint (and store-summary, if present) is the rightmost such entry, so it will be used at the next step.
    \end{enumerate}

    \begin{figure}[ht]
        \centering
        \resizebox{0.82\linewidth}{!}{
            \input{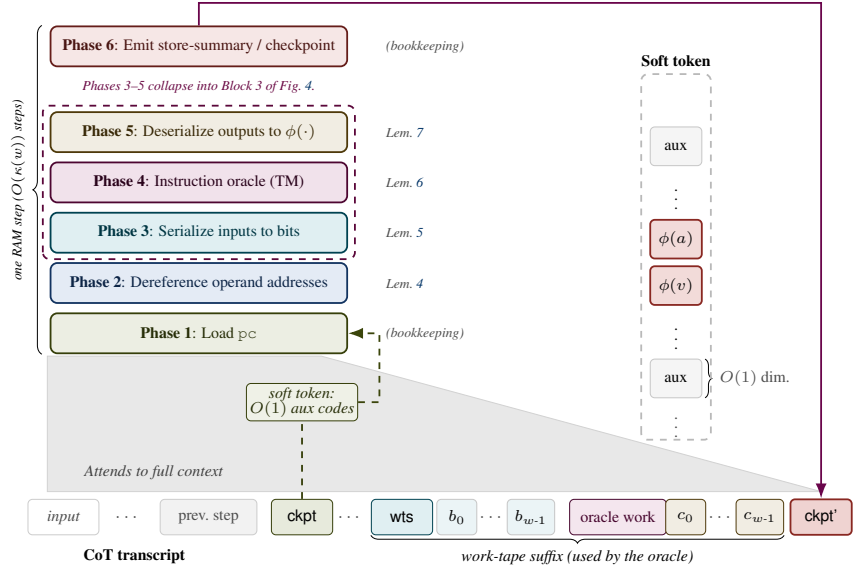}
        }
        \caption{Six-step CCoT simulation of a RAM instruction; detailed caption in Appendix, \Cref{fig:ccot-phase-blocks}.}
        \label{fig-preview:ccot-phase-blocks}
    \end{figure}

    Correctness follows from induction on $s$:
    step~(5) appends the new rightmost checkpoint encoding the updated registers/PC, and appends the new rightmost store-summary for any updated address.
    The total number of decoding steps spent in phases~(2)--(4) is $\bigO(w + \instrCostFun{\wordsize})$ per RAM instruction, yielding $\bigO(t\,\instrCostFun{\wordsize})$ steps overall. The following lemmas implement each step (proofs in \Cref{app:ccot}).
    
    \begin{restatable}[Steps 1--2: loading values]{lemma}{dereferencingLemmaCCoT}\label{lem:ccot-dereferencing}
        Suppose the transcript contains store-summary positions, each tagged and whose soft token stores a pair $(\phi(a_j),\phi(v_j))$.
        Let $\tau$ be a transcript position at which a query address $a$ is available as $\phi(a)$ in the residual stream, and suppose $a\le \tau$.
        Then a log-precision, fixed-width continuous-CoT transformer can retrieve $\phi(v_{j^*})$, where $j^*$ is the largest index with $a_{j^*}=a$.
        If no such $j^*$ exists, it retrieves $\phi(0)$.
    \end{restatable}
    
    \begin{restatable}[Step 3: serialization]{lemma}{serializationLemmaCCoT}\label{lem:ccot-serialization}
        Let $\tau$ be a transcript position at which a value $v$ is available as $\phi(v)$ in the residual stream, and suppose $v \leq \tau$.
        Then a log-precision, fixed-width continuous-CoT transformer can emit the $\wordsize$ bits of $v$ over the next $\wordsize$ decoding steps, ordered from least to most significant, while keeping $\phi(v)$ available in designated coordinates of the soft token throughout.
    \end{restatable}
    
    \begin{restatable}[Step 4: instruction execution]{lemma}{instructionExecutionLemmaCCoT}\label{lem:ccot-instruction-execution}
        Fix any RAM instruction $I$ on $\wordsize$-bit operands whose bit-serial implementation runs in $\instrCostFun{\wordsize}$ decoding steps.
        There exists a log-precision, fixed-width continuous-CoT transformer subnetwork that, once activated, simulates this bit-serial computation for $\bigO(\instrCostFun{\wordsize})$ decoding steps and emits the output bit streams of $I$ onto the transcript.
        Moreover, the subnetwork can be made to ignore the rest of the transcript by delimiting a \emph{work-tape suffix} (with a start marker) and restricting all attention inside the subroutine to positions in that suffix.
    \end{restatable}
    
    \begin{restatable}[Step 5: deserialization]{lemma}{deserializationLemmaCCoT}\label{lem:ccot-deserialization}
        There exists a log-precision, fixed-width continuous-CoT transformer that, on bit sequence $b_0,\ldots,b_{\wordsize-1}$ followed by $\wordsize$ arbitrary tokens, computes $\phi(v)$ in the residual stream at the last step, where $v=\sum_{k=0}^{\wordsize-1} b_k 2^k$ is the LSB evaluation of $b$.
    \end{restatable}
    
    \textbf{Step 6: storing new values.}
    We emit a soft token encoding the target position to write to (from load step) as $k$ and the output value (from deserialize) as $v$. This completes the store step, satisfying our inductive invariant for step $s+1$, which completes the proof.
\end{proof}

We have the following consequences of \Cref{thm:wordram-ccot}, depending on the complexity of the instruction set:

\begin{corollary}[Simulating Word RAMs with CCoT]\label{cor:wordram-ccot}
    Let $M$ and $T$ be as in \Cref{thm:wordram-ccot}. Depending on the complexity of the instruction set $\instrCostFun{\wordsize}$ of $M$, the following hold:
    \begin{enumerate}[label=(\roman*),leftmargin=20pt,noitemsep,topsep=0pt]
        \item $T$ can simulate $t$ steps of $M$ with arbitrary $\textrm{poly}(\wordsize)$-time instructions using $\bigOtilde(t)$ steps.
        
        \item $T$ can simulate $t$ steps of $M$ with \emph{full-arithmetic} flat instructions using $\bigO(t \log^2 t)$ steps.
        
        \item $T$ can simulate $t$ steps of $M$ with \emph{simple-arithmetic} flat instructions using $\bigO(t \log t)$ steps.
    \end{enumerate}
\end{corollary}

\section{Simulating \wRAM{}s with Hybrid Models} \label{sec:hybrid}

Finally, we show that, similarly to using CCoT, leveraging hybrid models that mix transformer and linear RNN components allow simulating a \wRAM{}:

\begin{theorem}[Simulating Word RAMs with Hybrid Models]\label{thm:wordram-hybrid}
    Let $M$ be a $c$-dimensional-array \wRAM with word size $w = \log (n + t) + \bigO(1)$ and maximum instruction cost $\instrCostFun{\wordsize}$.
    Then there exists a log-precision, fixed-width hybrid model $H$ with a single RNN layer at the bottom such that, for any $t \geq n$, $t$ steps of $M$ can be simulated by $H$ using $\bigO(t \instrCostFun{\wordsize} )$ steps.
\end{theorem}

Let $v = \sum_{k=0}^{w-1} b_k 2^k$.
The proof of this theorem is structurally identical to that of \Cref{thm:wordram-ccot} for transformers with \emph{continuous} CoT, except for the way values $v$, represented as $\phi(v)$ in the residual stream, are serialized as bits $b_0, b_1, \ldots, b_{w-1}$ via CoT, and the way bits $b_0, b_1, \ldots, b_{w-1}$ in the input layer are deserialized as $\phi(v)$ in the residual stream. Since we no longer have access to soft tokens in the CoT output, we cannot simply emit intermediate $\phi$ representations as part of the CoT to condition upon in the next step. However, we can use a similar ability afforded by RNNs, which naturally allow conditioning on values available in the residual stream at the immediately previous position.

The only other difference is in the \emph{store} step, where the continuous CoT captures operations of the form $\terminal{mem}[a]\leftarrow v$ by emitting a store-summary soft token storing $(\phi(a),\phi(v))$. Instead, we now serialize $a$ and $v$ using $w$ bits each via the (updated) serialization lemma below, and use the (updated) deserialization lemma on these bits to obtain $\phi(a)$ and $\phi(v)$ in the layer $1$ residual stream after $w$ additional dummy tokens.

All other steps in the construction for \Cref{thm:wordram-ccot} remain unchanged. In the remainder of this section, we thus focus on deriving counterparts of \Cref{lem:ccot-serialization,lem:ccot-deserialization} for hybrid architectures.
The serialization construction is essentially the same as in the proof of \Cref{lem:ccot-serialization}, except that without continuous CoT, we cannot emit $\phi(v_{k+1})$, defined as $\phi(\lfloor v_k / 2 \rfloor)$, at position $\tau + k + 1$ and read it back in at the following position. Instead, we emit $b_k$, computed via \Cref{lem:division}. We then compute $\phi(\floor{v_k / 2})$ from $\phi(v_k)$ and $b_k$ via its equivalent form $\phi(\frac{1}{2} (v_k - b_k))$ by leveraging \Cref{prop:phi-affine} to turn it into a linear transformation that an RNN head can perform. This gives the following result (proof in \Cref{app:hybrid}):
\begin{restatable}[Serialization]{lemma}{serializationLemmaHybrid} \label{lem:hybrid-serialization}
    Let $\tau$ be a transcript position at which a value $v$ is available as $\phi(v)$ in the residual stream, and suppose $v \leq \tau$.
    Then a log-precision, fixed-width hybrid model can emit the $w$ bits of $v$ over the next $w$ decoding steps, ordered from least to most significant.
\end{restatable}
  
For deserialization, we use the same recursive computation as in the corresponding lemma for continuous CoT (\Cref{lem:ccot-deserialization}), except now the computation is within the hidden state of the linear RNN rather than via soft tokens. This gives the following result (proof in \Cref{app:hybrid}):

\begin{restatable}[Deserialization]{lemma}{deserializationLemmaHybrid} \label{lem:hybrid-deserialization}
    There exists a log-precision, fixed-width hybrid model that, on any bit sequence $b_0, \ldots, b_{w-1}$ followed by $w$ arbitrary tokens, computes $\phi(v)$, where $v = \sum_{k=0}^{w-1} b_k 2^k$ is the LSB evaluation of $b$, in the residual stream at depth $1$ at the last token.   
\end{restatable}

Analogously to \Cref{cor:wordram-ccot} for the CCoT simulation, we have the following notable cases:

\begin{corollary}[Simulating Word RAM with Hybrid Models]\label{cor:wordram-hybrid}
    Let $M$ and $H$ be as in \Cref{thm:wordram-hybrid}. Depending on the complexity of the instruction set $\instrCostFun{\wordsize}$ of $M$, the following hold:
    \begin{enumerate}[label=(\roman*),leftmargin=20pt,noitemsep,topsep=0pt]
        \item $H$ can simulate $t$ steps of $M$ with arbitrary $\textrm{poly}(\wordsize)$-time instructions using $\bigOtilde(t)$ steps.
        
        \item $H$ can simulate $t$ steps of $M$ with \emph{full-arithmetic} flat instructions using $\bigO(t \log^2 t)$ steps.
        
        \item $H$ can simulate $t$ steps of $M$ with \emph{simple-arithmetic} flat instructions using $\bigO(t \log t)$ steps.

    \end{enumerate}
\end{corollary}

\section{Conclusion}

We show that poly-logarithmic-width transformers \upd{with rightmost-unique hard attention} can simulate a \wRAM. While it is unclear whether \upd{standard} fixed-width transformers can also do so, we show that they can as long as we allow one of two minimal extensions: continuous CoT transformers, \upd{which have been studied before,} and hybrid architectures, \upd{which are now common in many frontier models}.
Our results show that \upd{these extensions do allow transformers to simulate} algorithms with their general textbook runtimes up to a polylogarithmic factor overhead, removing the quadratic overhead that would be incurred by going through previously established Turing completeness results \citep{perez2021attention,merrill2024cot}.
\upd{Thus, transformers can operate at a much more efficient and natural level of abstraction of algorithmic computation than the low-level Turing machine model.}
It remains open to what degree these extensions are necessary for efficient \wRAM simulation,
or benefit algorithmic reasoning in transformer models in practice.

\section*{Limitations}
\upd{Our analysis focused on the the ability of transformers and hybrid models to efficiently express a \wRAM{} simulation, not the learnability of this simulation from data. Our positive results are for three minimal extensions of transformers; whether a \wRAM{} simulation can be achieved without these extensions remains open. Our study focused on the theoretical expressivity aspects of these models, motivating but leaving open a parallel empirical study.
}



\bibliographystyle{plainnat}
\bibliography{references}

\newpage


\appendix

\section{Details on the Preliminaries} \label{app:preliminaries}

\subsection{The Transformer Model}\label{app:transformer-model}

This appendix restates the decoder-only transformer model used in
\citet{merrill2024cot}, which itself builds on the component-level
formalization of \citet{merrill2023logic}.
Let $\alphabet$ be the fixed token alphabet.
Let $\mathbb D_p$ be the datatype of $p$-bit floating-point numbers with
truncated arithmetic as in \citet{merrill2023parallelism}.

\begin{definition}[Decoder-only transformer]
    A $p$-precision decoder-only transformer with $h$ heads, $d$ layers,
    model dimension $m$, and feedforward width $u$ consists of:
    (i) an embedding map $e:\alphabet \times \mathbb \inputlen\to \mathbb D_p^m$;
    (ii) for each layer $\ell$ and head $k$, a similarity map
$s_k^\ell:\mathbb D_p^m\times\mathbb D_p^m\to\mathbb D_p$
    and a value map
$v_k^\ell:\mathbb D_p^m\to\mathbb D_p^{m/h}$;
    (iii) for each layer $\ell$, an activation map
$f^\ell:(\mathbb D_p^{m/h})^h\times\mathbb D_p^m\to\mathbb D_p^m$;
    (iv) an output map $\gamma:\mathbb D_p^m\to\alphabet'$.
\end{definition}

\begin{definition}[One strict-causal decoding step]
    For a prefix $z=z_1\cdots z_n$, set $\mathbf h_i^0=e(z_i,i)$.
    For each layer $\ell$ and head $k$,
    \upd{first define $w_{i, k}^\ell$ as the attention weights determined from $s_{i, k}^\ell$ according to the attention type (discussed shortly). Next, define}
\[
    \mathbf a_{i,k}^{\ell}
    =
    \sum_{j<i}
    w_{j,k}^\ell(\mathbf h_i^{\ell-1},\mathbf h_j^{\ell-1})
    \,v_k^\ell(\mathbf h_j^{\ell-1}) .
\]
    Then
\[
    \mathbf h_i^\ell
    =
    f^\ell(\mathbf a_{i,1}^\ell,\ldots,\mathbf a_{i,h}^\ell,\mathbf h_i^{\ell-1}),
    \qquad
    g(z)=\gamma(\mathbf h_n^d).
\]
\end{definition}

\upd{Each of our results is for one of two types of attention \citep[cf.][]{strobl-etal-2024-formal,hao-etal-2022-formal}: rightmost-unique hard attention (UHAT) for fixed-precision models and averaging hard attention (AHAT) for log-precision and fixed-width models.
Under AHAT, we set $w^\ell_{i,k} = 1 / m$ at any position $j$ where $s^\ell_{i, k}$ is maximized (and 0 otherwise), where $m$ is the number of such positions.
Under UHAT, we simply set $w^\ell_{i,k} = 1$ at the largest $j < i$ where $s^\ell_{i, k}$ is maximized, and $0$ otherwise.
}

\begin{definition}[Projected pre-norm]
    For a residual vector $\mathbf h$ and a linear map $\mathbf M$,
    projected pre-norm applies layer norm to $\mathbf M\mathbf h$:
\[
    \PLN(\mathbf h;\mathbf M)=\LN(\mathbf M\mathbf h).
\]
\end{definition}

We note that now-standard peri-norm \citep{kim2025periln} or reordered norm \citep{olmo2026olmo3}, which did not exist at the time of \citet{merrill2024cot}, allows implementing projected pre-norm constructions with no extra projection.
With reordered norm, we can simply always write normalized values at the layer output and then apply a projection to select those values when needed.
With peri-norm, the same idea works: all writes to the residual stream are normalized, so when we read a normalized vector at the next sublayer input, the number of writes gets normalized away with no distortions.

As in \citet{merrill2024cot}, the lower-bound construction
does not depend on any special positional encoding.
We assume standard causal masking rather than strict causal masking.

\subsection{Layernorm Hash Details} \label{app:layernorm}

\phiEqualityAndGapProp*

\begin{proof}
    Since $\phi(x)=\frac{1}{\sqrt{2(x^2+1)}}(x,1,-x,-1)$, one computes
\[
    \langle \phi(x),\phi(y)\rangle = \frac{xy+1}{\sqrt{(x^2+1)(y^2+1)}}
\]
    and therefore
\[
    1-\langle \phi(x),\phi(y)\rangle^2 = \frac{(x-y)^2}{(x^2+1)(y^2+1)}.
\]
    If $u\neq v$ are integers in $[0,T]$, then $(u-v)^2\ge 1$ and $(u^2+1)(v^2+1)\le (T^2+1)^2$, hence
$1-\langle \phi(u),\phi(v)\rangle^2 \ge 1/(T^2+1)^2$.
    Using $\sqrt{1-\epsilon}\le 1-\epsilon/2$ gives
$\langle \phi(u),\phi(v)\rangle \le 1-\frac{1}{2(T^2+1)^2}$; the stated $\delta_T$ is a slackened bound.
\end{proof}

\affineUpdateProp*

\begin{proof}
    Noting that $(s_1, s_2, -s_1, -s_2)$ is simply $\phi(s_1, s_2)$ and that $\phi(v) = \phi(v,1)$, scale invariance of $\phi$ and the precondition in the proposition implies $v = \frac{s_1}{s_2}$. Thus $av + b = a \frac{s_1}{s_2} + b$. It follows that $\phi(av + b) = \phi(a \frac{s_1}{s_2} + b) = \phi(a \frac{s_1}{s_2} + b, 1) = \phi(a s_1 + b s_2, s_2)$, where the last equality follows from scale invariance of $\phi$.
\end{proof}

\section{Proofs for Restricted Word RAMs} \label{app:cdim-wram}

\cDimwRAMLemma*
\begin{proof}
    We will show how to simulate $M$ with $M'$.
    The following invariant will be helpful:
    $M$ has size $w = c w'$, so any value $x$ can represented with $c$ values $x'_1, \ldots, x'_c$, each of word size $w'$.
    
    For any instruction in $M$ that retrieves the value at $x$ and stores it into $y$, we will instead have $c$ instructions that retrieve the value at each $x'_i$ and store it into $y'_i$ for $1 \leq i \leq c$.
    Similarly, whenever we write the value at $y$ into $x$, we will have $c$ instructions that write value $y'_i$ into $x'_i$.
    Thus, each read or write instruction in $M$ can be simulated by $c$ instructions in $M'$.
    
    It remains to be shown that arithmetic instructions in $M$ can also be simulated in $M'$:
    \begin{itemize}[leftmargin=20pt]
        \item \emph{Adding} $x$ and $y$ in $M$ can be implemented in $M'$ by starting with $\kappa_1 = 0$ and adding blocks $x'_i + y'_i + \kappa_i$ iteratively for $i = 1, 2, \ldots, c$, with $\kappa_{i+1} = \indFun{x' + y' \geq 2^w}$ being the carry bit for the subsequent step. This takes at most $2c$ steps.
        
        \item \emph{Subtracting} mirrors addition, with the \emph{borrow} bit $\beta_{i+1} = \indFun{x' + \beta_i < y'}$.
        (If the final borrow bit $\beta_{c+1} = 1$, the result is negative.)
        
        \item \emph{Multiplying} $x$ and $y$ in $M$ can be implemented in $M'$ via long multiplication in base $m = 2^{w'}$.
        Write $x = \sum_{i=1}^c x'_i \cdot m^{i-1}$ and $y = \sum_{j=1}^c y'_j \cdot m^{j-1}$.
        For each $j = 1, \ldots, c$, we compute the partial product $a_j = x \cdot y'_j \cdot m^{j-1}$ and keep a running sum $S \leftarrow S + a_j$.
        The partial product $x \cdot y'_j$ is itself computed block-by-block: for each $i = 1, \ldots, c$, multiply $x'_i \cdot y'_j$ (a product of two $w'$-bit numbers, fitting in $2w'$ bits) and add it into block $i$ of the partial product together with any carry from block $i-1$.
        Since each $x'_i, y'_j < 2^{w'}$, the product $x'_i \cdot y'_j < 2^{2w'}$, so the carry into block $i+1$ is at most $2^{w'}-1$, which fits in $w'$ bits.
        Computing one partial product $x \cdot y'_j$ thus takes $O(c)$ steps; summing it into $S$ (via the addition procedure above) takes another $O(c)$ steps.
        Over all $c$ values of $j$, the total cost is $O(c^2)$ steps.
        
        \item \emph{Dividing} $x$ by $y$ in $M$ can be implemented in $M'$ via long division in base $m = 2^{w'}$, analogously to multiplication \citep[][\S 4.3.1, Alg. D]{10.5555/270146}.
        The quotient $\lfloor x/y \rfloor$ has $c$ base-$m$ digits $q'_1, \ldots, q'_c$; we recover them one by one from most to least significant.
        At each step, the top two blocks of the running remainder and the top block of $y$ together determine the next quotient digit exactly via a single $w'$-bit division (available as a primitive in $M'$).
        Each of the $c$ digits thus costs $\bigO(c)$ steps (one scalar division plus a multiply-and-subtract to update the remainder), giving $\bigO(c^2)$ total.
        
        \item \emph{Taking mod}: the remainder $r = x - \lfloor x/y \rfloor \cdot y$ is computed alongside the quotient in the procedure above at no extra asymptotic cost.
        
        \item \emph{Left or right shifting}. We describe the left shift case; the right shift case is analogous.
        Suppose we want to compute $z = x \ll k$.
        Decompose the shift amount as $k = q \cdot w' + r$ where $q = \lfloor k/w' \rfloor$ is the number of whole blocks to shift and $r = k \bmod w'$ is the intra-block shift.\footnote{Note that since $k < w = cw'$, the shift amount fits in a single $w'$-bit block of $M'$.
        Thus, $q = \lfloor k/w' \rfloor$ and $r = k \bmod w'$ are each single-block values computable in $\bigO(1)$ steps with primitives in $M'$.}
        Thus, shifting left by $k$ bits moves block $i$ of the input to block $i + q$ of the output, with an $r$-bit internal shift.
        Concretely, for $i = 1, \ldots, c$:
        \[
        z_i' =
        \begin{cases}
            0 & \text{if } i \leq q, \\
            (x_{i-q}' \ll r) \mathbin{|} (x_{i-q-1}' \gg (w' - r)) & \text{if } i > q \text{ and } i - q - 1 \geq 1, \\
            x_{i-q}' \ll r & \text{if } i > q \text{ and } i - q - 1 < 1,
        \end{cases}
        \]
        where all individual shifts are within $[0, w')$ so they stay inside word size $w'$.
        Each output block depends on at most 2 input blocks, so the whole operation takes $\bigO(c)$ instructions.
        
        \item \emph{Comparing} $x$ with $y$ in $M$ can be implemented in $M'$ by iteratively comparing $x'_i$ with $y'_i$. This takes at most $c$ steps.
        
        \item \emph{Branching} between $x$ and $y$ in $M$ based on a condition can be implemented by a cascade of $c$ branches on each block of $x, y$, which can be done in at most $c$ steps.
    \end{itemize}
    Thus, we can simulate each instruction of $M$ with a constant number of instructions in $M'$.
\end{proof}

\simplifiedDivisionLemma*

\begin{proof}
    The idea is the same as in the more general proof of \citet{merrill2025little} for division by an arbitrary small value $m$. In their construction, each position $i$ computes $\phi(i/m)$ and performs an equality test with $\phi(j)$ computed at all prior positions $j < i$ to determine if $i/m$ is an integer, i.e., $i$ is a multiple of $m$. Then position $t$ attends back to find the closest multiple of $m$ via rightmost hard attention to compute the mod and div values.
    
    In our special use case, we can simplify this process as follows: each position $j < t$ computes a key $k_j = \phi(2j)$ and 
    value $\phi(j)$. Position $t$ uses query $q_t = \phi(v-1)$, which can be computed via \Cref{prop:phi-affine}. It can be verified that the closest matching key is $k_{\ceil{(v-1)/2}}$,\footnote{Suppose $v$ is odd, i.e., $v = 2j^* + 1$ for some integer $j^*$. Then $(v-1)/2 = j^*$ and we have an exact key-query match between $q_t$ and $k_{j^*}$. When $v$ is even, say $v = 2j^*$, then $(v-1)/2 = j^* - 1/2$. The two contenders for the closest match with $q_t = \phi(2j^* - 1)$ in this case are $k_{j^* - 1} = \phi(2j^* - 2)$ and $k_{j^*} = \phi(2j^*)$. Since these are unit norm vectors, the key-query inner product is the same as cosine distance, which is larger between $\phi(2j^* - 1)$ and $\phi(2j^*)$ than between $\phi(2j^* - 1)$ and $\phi(2j^* - 2)$.} and the corresponding value is $\phi(\ceil{(v-1)/2})$. Observe that for all integer values $v$, $\ceil{(v-1)/2} = \floor{v/2}$. Thus the retrieved value is $\phi(\floor{v/2})$, which we add to the residual stream. To obtain $\phi(v \bmod 2)$, we compute $\phi(2 \floor{v/2})$ via \Cref{prop:phi-affine} and perform an equality test with $\phi(v)$. If these are equal, $\phi(v \bmod 2)$ is $0$, otherwise it is $1$; we add this to the residual stream as well.
    %
\end{proof}

\section{Proofs of Polylog-Width Transformer Results} \label{app:logwidth}

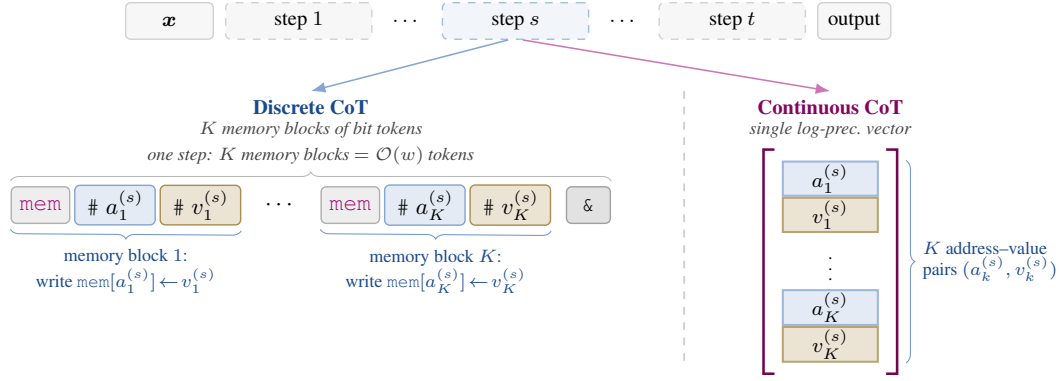
\begin{figure}[t]
    \centering
    \colorlet{AddrFill}{ETHBlue!12}
    \colorlet{AddrBorder}{ETHBlue!55}
    \colorlet{StateFill}{ETHGreen!15}
    \colorlet{StateBorder}{ETHGreen!65}
    \colorlet{StoreFill}{ETHBronze!18}
    \colorlet{StoreBorder}{ETHBronze!75}
    \colorlet{CkptFill}{ETHGray!20}
    \colorlet{CkptBorder}{ETHGray!65}
    \colorlet{InitFill}{ETHPetrol!10}
    \colorlet{InitBorder}{ETHPetrol!55}
    \colorlet{KeyFill}{ETHGray!12}
    \colorlet{KeyBorder}{ETHGray!55}
    \resizebox{\linewidth}{!}{%
    \begin{tikzpicture}[
        font=\small,
        node distance=1.5mm,
        field/.style={draw, rounded corners=2pt, minimum height=5mm, align=center,
        inner xsep=3pt, inner ysep=2pt, line width=0.55pt},
        key/.style  ={field, fill=KeyFill,   draw=KeyBorder},
        addr/.style ={field, fill=AddrFill,  draw=AddrBorder},
        state/.style={field, fill=StateFill, draw=StateBorder},
        store/.style={field, fill=StoreFill, draw=StoreBorder},
        ckpt/.style ={field, fill=CkptFill,  draw=CkptBorder},
        init/.style ={field, fill=InitFill,  draw=InitBorder},
        ghost/.style={field, fill=gray!6,    draw=gray!35, dashed},
        brace/.style={decorate, decoration={brace, amplitude=3.5pt, raise=2pt}},
        ]
        \node[field, fill=gray!6, draw=gray!35, minimum width=12mm] (inp)  {\footnotesize $\inputstr$};
        \node[ghost, minimum width=20mm, right=1.5mm of inp]       (b1)   {\footnotesize step $1$};
        \node[right=1.5mm of b1, font=\footnotesize]                (bdot) {$\cdots$};
        \node[ghost, minimum width=20mm, right=1.5mm of bdot, fill=ETHBlue!6, draw=ETHBlue!40]
        (bs)   {\footnotesize step $s$};
        \node[right=1.5mm of bs, font=\footnotesize]                (bdot2){$\cdots$};
        \node[ghost, minimum width=20mm, right=1.5mm of bdot2]      (bt)   {\footnotesize step $\ramtime$};
        \node[field, fill=gray!6, draw=gray!35, minimum width=10mm, right=1.5mm of bt]
        (out)  {\footnotesize output};

        \coordinate (zoombase) at ([yshift=-20mm]bs.south);

        \node[key,   minimum width=8mm,  anchor=north] (kk1) at ([xshift=-65mm]zoombase)
        {\memkey};
        \node[addr,  minimum width=11mm, right=0.5mm of kk1] (ka1)
        {\texttt{\#} $a_1^{(s)}$};
        \node[store, minimum width=11mm, right=0.5mm of ka1] (kv1)
        {\texttt{\#} $v_1^{(s)}$};
        \node[draw=none, minimum width=5mm, right=2mm of kv1] (kdots) {$\cdots$};
        \node[key,   minimum width=8mm,  right=2mm of kdots] (kkK) {\memkey};
        \node[addr,  minimum width=11mm, right=0.5mm of kkK] (kaK)
        {\texttt{\#} $a_K^{(s)}$};
        \node[store, minimum width=11mm, right=0.5mm of kaK] (kvK)
        {\texttt{\#} $v_K^{(s)}$};
        \node[ckpt,  minimum width=6mm,  right=2mm of kvK] (keoi)
        {\EOI};

        \draw[brace, decoration={brace, amplitude=3.5pt, raise=2pt, mirror}, ETHBlue!65]
        (kk1.south west) -- (kv1.south east)
        node[midway, below=6pt, font=\scriptsize, text=ETHBlue!80!black, align=center]
        {memory block $1$:\\write $\terminal{mem}[a_1^{(s)}]\!\leftarrow\! v_1^{(s)}$};
        \draw[brace, decoration={brace, amplitude=3.5pt, raise=2pt, mirror}, ETHBlue!65]
        (kkK.south west) -- (kvK.south east)
        node[midway, below=6pt, font=\scriptsize, text=ETHBlue!80!black, align=center]
        {memory block $K$:\\write $\terminal{mem}[a_K^{(s)}]\!\leftarrow\! v_K^{(s)}$};

        \draw[brace, ETHGray!50]
        (kk1.north west) -- (keoi.north east)
        node[midway, above=5pt, font=\scriptsize\itshape, text=ETHGray!60!black]
        {one step: $K$ memory blocks $=\bigOFun{\wordsize}$ tokens};

        \coordinate (sepx) at ([xshift=10mm]keoi.east);

        \node[draw=AddrBorder!80, fill=AddrFill, line width=0.8pt,
        minimum width=13mm, minimum height=4.5mm, inner xsep=2pt, inner ysep=1pt,
        align=center, anchor=north]
        (va1) at ([xshift=30mm, yshift=3.5mm]keoi.east |- kk1.north)
        {\footnotesize $a_1^{(s)}$};
        \node[draw=StoreBorder!80, fill=StoreFill, line width=0.8pt,
        minimum width=13mm, minimum height=4.5mm, inner xsep=2pt, inner ysep=1pt,
        align=center, below=0pt of va1] (vv1)
        {\footnotesize $v_1^{(s)}$};
        \node[draw=none, minimum width=13mm, minimum height=3.5mm,
        below=0pt of vv1, font=\scriptsize, align=center] (vdots) {$\vdots$};
        \node[draw=AddrBorder!80, fill=AddrFill, line width=0.8pt,
        minimum width=13mm, minimum height=4.5mm, inner xsep=2pt, inner ysep=1pt,
        align=center, below=0pt of vdots] (vaK)
        {\footnotesize $a_K^{(s)}$};
        \node[draw=StoreBorder!80, fill=StoreFill, line width=0.8pt,
        minimum width=13mm, minimum height=4.5mm, inner xsep=2pt, inner ysep=1pt,
        align=center, below=0pt of vaK] (vvK)
        {\footnotesize $v_K^{(s)}$};

        \coordinate (vtl) at ([xshift=-2.5mm, yshift=1.5mm]va1.north west);
        \coordinate (vbl) at ([xshift=-2.5mm, yshift=-1.5mm]vvK.south west);
        \coordinate (vtr) at ([xshift= 2.5mm, yshift=1.5mm]va1.north east);
        \coordinate (vbr) at ([xshift= 2.5mm, yshift=-1.5mm]vvK.south east);
        \draw[ETHPurple!70!black, line width=1pt]
        ([xshift=1.5mm]vtl) -- (vtl) -- (vbl) -- ([xshift=1.5mm]vbl);
        \draw[ETHPurple!70!black, line width=1pt]
        ([xshift=-1.5mm]vtr) -- (vtr) -- (vbr) -- ([xshift=-1.5mm]vbr);

        \draw[brace, ETHBlue!65, decoration={brace, amplitude=3.5pt, raise=2pt}]
        ([xshift=3mm]va1.north east) -- ([xshift=3mm]vvK.south east)
        node[midway, right=5pt, font=\scriptsize, text=ETHBlue!80!black, align=left]
        {$K$ address--value\\pairs $(a_k^{(s)}, v_k^{(s)})$};

        \coordinate (lmid) at ($(kk1.west)!0.5!(keoi.east)$);

        \node[font=\footnotesize\bfseries, text=ETHBlue!80!black,
        anchor=north, align=center]
        (lhdr) at ([yshift=-7mm]bs.south -| lmid)
        {Discrete CoT\\[-2pt]{\normalfont\scriptsize\color{ETHGray!65!black}%
        \itshape $K$ memory blocks of bit tokens}};

        \node[font=\footnotesize\bfseries, text=ETHPurple!80!black,
        anchor=north, align=center]
        (rhdr) at ([yshift=-7mm]bs.south -| va1)
        {Continuous CoT\\[-2pt]{\normalfont\scriptsize\color{ETHGray!65!black}%
        \itshape single log-prec.\ vector}};

        \draw[-{Latex[length=2mm,width=1.5mm]}, ETHBlue!55, line width=0.6pt]
        ([xshift=-1mm]bs.south) -- (lhdr.north);
        \draw[-{Latex[length=2mm,width=1.5mm]}, ETHPurple!55, line width=0.6pt]
        ([xshift=1mm]bs.south) -- (rhdr.north);

        \draw[ETHGray!35, dashed, line width=0.7pt]
        (sepx |- lhdr.north) -- (sepx |- vvK.south);
    \end{tikzpicture}}%
    \caption{%
    \textbf{Executing a \wRAM{} program with discrete vs.\ continuous CoT under the unified-memory convention.}
    The entire machine state---program counter, registers, and memory---lives in a single random-access memory $\terminal{mem}$, so each simulated step is fully described by the constant-many ($K$) address--value pairs it writes.
    \emph{Bottom-left (discrete CoT):} each step emits $K$ \emph{memory blocks} $\memkey\,\BOV\,\bits(a_k^{(s)})\,\BOV\,\bits(v_k^{(s)})$ followed by an \EOI{} marker, totaling $\bigOFun{\wordsize}$ bit tokens.
    \emph{Bottom-right (continuous CoT):} the same $K$ address--value pairs are packed as coordinates of a \emph{single} log-precision vector emitted at step $s$.
    }
    \label{fig:block-detailed}
\end{figure}

\paragraph{Serialization format.}
Before stating the propositions, we fix the transcript format used throughout all four components.

\textbf{Alphabet.}
The output alphabet is $\alphabet \sqcup \{\mathtt{0},\mathtt{1},\BOV, \EOI, \memkey\}$, where $\EOI$ is the \emph{end-of-step} marker, $\memkey$ is the \emph{memory-write} marker, and $\mathtt{0},\mathtt{1}$ are \emph{bit tokens}.
The simulator uses no other key tokens: the program counter, registers, and any instruction-private intermediate state all live in $\terminal{mem}$ at fixed addresses (cf. \Cref{sec:wordram}), so every committed write has the same shape.

\textbf{Memory blocks.}
Every write to the simulated memory appends one \emph{memory block} to the transcript:
\[
    \memkey
    \;\underbrace{\BOV\;a_1\,\cdots\,a_\wordsize}_{\text{address field}}
    \;\underbrace{\BOV\;v_1\,\cdots\,v_\wordsize}_{\text{value field}},
\]
occupying $2\wordsize + 3$ consecutive positions and recording $\terminal{mem}[a]\leftarrow v$.
\Cref{lem:reading-from-memory} matches an address $\bits(a)$ available in the residual stream against the address field of the rightmost $\memkey$ block, then uses the known offset $\wordsize+1$ to retrieve the value field.

\textbf{Steps.}
A \wRAM{} step is the run of (constant-many) memory blocks committed by one simulated instruction, terminated by an \EOI{} token; \EOI{} is therefore both an end-of-step marker and the anchor at which Blocks~1--3 produce the snapshot for the next step (see Block~4).

\textbf{Residual stream convention.}
Throughout the proofs, the phrase ``$\bits(v)$ is available in the residual stream at position $\inputidx$'' means that the $\wordsize$-dimensional vector $\bits(v) = (b_1, \ldots, b_\wordsize) \in \{0,1\}^\wordsize$ occupies a dedicated $\wordsize$-dimensional subspace of the $\bigTheta(\log^2\inputlen)$-dimensional residual stream at that position.
Each transformer layer reads from and writes to named subspaces; different subspaces hold different operands, intermediate addresses, and result slots so that they do not interfere.
This can be done without loss of generality \citep[][\S 4.2]{yang2026the}.

\textbf{State recovery.}
The contents of $\terminal{mem}[a]$ at any point in the transcript is the value of the rightmost memory block whose address field encodes $a$.
\Cref{lem:reading-from-memory} performs this lookup in constant depth, writing the retrieved $\bits(v)$ into a designated residual-stream subspace.\footnote{This is analogous to the commonly studied \emph{induction heads} \citep{olsson2022incontextlearninginductionheads}: a head first matches the address pattern, then retrieves the value at the position immediately following it.}

\textbf{Default values for unwritten cells.}
At the start of simulation no memory block has been emitted, so a rightmost-match read for any address returns $\vzero$.
The \wRAM{} program supplies a static initialization vector that fixes the initial value at every address used by the program (zero by default; specific constants for $\terminal{mem}[0]=\PCstart$ and any program-private scratch cells), and this initialization is baked into the transformer weights.
When a read for address $a$ has no matching block in the transcript (detected by a single attention head testing whether the rightmost match exists), an MLP overrides the default $\vzero$ output of \Cref{lem:reading-from-memory} with the program-specified initialization constant for that address.
Because the initialization map is fixed at compile time and depends only on the program, this override is a constant-depth, fixed-precision operation.

%
%
%
\begin{figure}[t]
\centering
\input{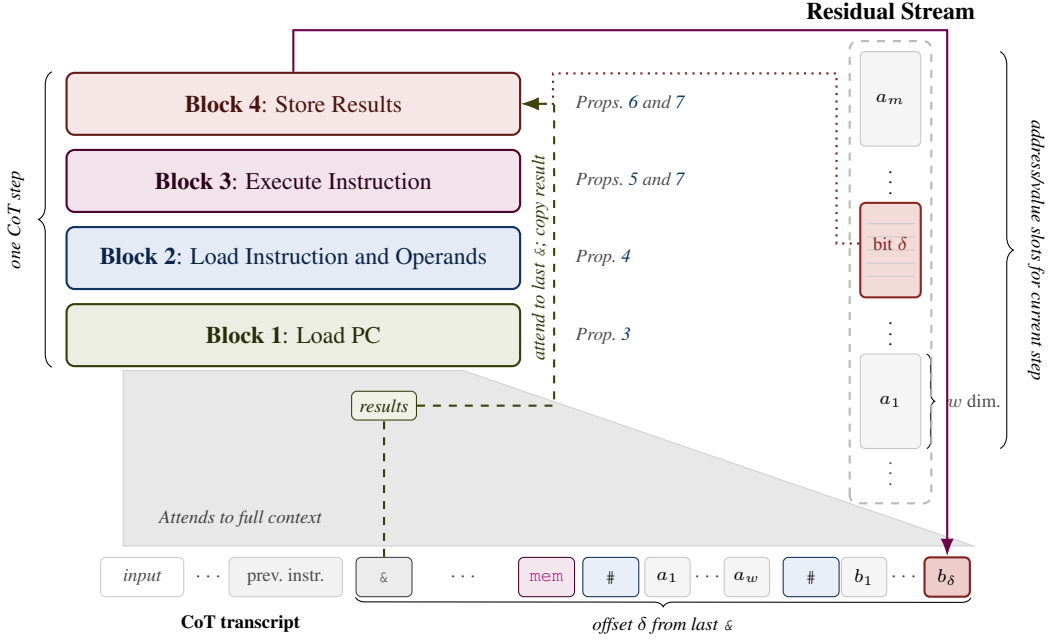}
\caption{\textbf{Layer-block structure and CoT transcript view} of one step in the $\log^2$-width simulation.
    \emph{Layer blocks (center):} each CoT step executes four constant-depth blocks.
    Block~1 reads $\bits(\terminal{mem}[0])$ (the \texttt{pc}) via one indexed memory read in two attention layers (\Cref{prop:logwidth-load-pc}).
    Block~2 uses two more sequential attention layers: the first identifies the active instruction and reads its compile-time operand addresses; the second resolves any runtime ($\terminal{mem}[a]$ or $\terminal{inp}[a]$) addresses (\Cref{prop:logwidth-dispatch}).
    Block~3 is a pure MLP that computes the destination value(s) and the next \texttt{pc} from the operands now in the residual stream (\Cref{prop:logwidth-execute}).
    Block~4 attends to the rightmost \EOI{} to compute the within-step offset $\delta$, copies the snapshot precomputed by Blocks~1--3 at that \EOI{}, and emits the appropriate $\delta^{\text{th}}$ token ($\memkey$, \BOV{}, address bit, value bit, or closing \EOI{}) into the transcript (\Cref{prop:logwidth-store}).
    \emph{CoT transcript (bottom):} the input and previous steps are followed by the \EOI{} that opens the current step (where Blocks~1--3 produced and cached the snapshot, depicted by the green badge), then the current step's memory-block serialization; $b_\delta$ (red) is the token emitted at this step; the shaded beam indicates that each block attends causally over the full prefix.
    \emph{Residual stream (right):} each memory block to be written by the current step occupies a pair of $\wordsize$-dimensional slots (address and value).
    Blocks~1--3 run at every position, but their rightmost-match reads only produce the correct snapshot at \EOI{} positions; at the $C\wordsize+D-1$ subsequent serialization positions their outputs are corrupted by the partially written current-step memory block and are disregarded by Block~4, which instead copies the snapshot from the most recent \EOI{}.}
\label{fig:logwidth-layer-blocks}
\end{figure}

We now outline the proof of \cref{thm:logwidthSimulation}, which is the main result of this section; the propositions that follow will fill in the details of each component.

\logwidthSimulation*
\begin{proof}
    \textbf{Memory-only state.}
    Following the \wRAM{} definition (\Cref{sec:wordram}), the entire machine state---program counter, registers, and the memory---is stored as a single random-access memory $\terminal{mem}$ indexed by integer addresses.
    The transformer simulates the \wRAM{} by maintaining, in the CoT transcript, a log of memory writes from which the current contents of any address can be recovered.
    Concretely, every write to an address $a$ is appended to the transcript as a \emph{memory block}
    \[
        \memkey \;\BOV\; \bits(a)\;\BOV\;\bits(v),
    \]
    where $\bits(\cdot)$ is the binary encoding into $\{0,1\}^\wordsize$.
    The contents of $\terminal{mem}[a]$ at any point in the transcript is the value $v$ of the rightmost memory block whose address field encodes $a$, or the program-supplied initialization constant if no such block exists; \Cref{lem:reading-from-memory} shows that a rightmost \uhat{} can perform this lookup in constant depth.
    Read-only input is handled identically: $\terminal{inp}[i]$ is recovered by direct positional attention to input position $i$.
    The transcript is partitioned by \EOI{} markers into \emph{steps}, each step being the run of memory blocks committed by one simulated \wRAM{} instruction; see \Cref{fig:logwidth-layer-blocks}.

    \textbf{What one CoT step does.}
    The construction decomposes each CoT step into four consecutive \emph{layer blocks}, each handled by a constant-depth subnetwork.

    \begin{description}[leftmargin=0pt,itemsep=2pt]
        \item[\textbf{Block 1: Load \texttt{pc}} (\Cref{prop:logwidth-load-pc}).]
        A single indexed read (\Cref{lem:reading-from-memory}) with query address $0$ retrieves $\bits(\terminal{mem}[0])$ from the transcript, recovering the current \texttt{pc} value into a dedicated residual-stream subspace.
        The indexed read uses two attention layers (address matching against the rightmost $\memkey$ block with $a = 0$, then value retrieval).

        \item[\textbf{Block 2: Load Instruction and Operands} (\Cref{prop:logwidth-dispatch}).]
        Each instruction $I_p$ in the program declares a fixed \emph{operand descriptor}---a static list of memory addresses it reads---which the program supplies and is hard-coded into the transformer weights.
        The list contains two kinds of addresses: \emph{compile-time} addresses (the registers $-i$ and any program-private scratch cells used by $I_p$) and \emph{runtime} addresses (memory or input cells whose index $\terminal{mem}[a]$ or $\terminal{inp}[a]$ is itself the value of another address that must be read first).

        \emph{Sub-step 2a (dispatch and compile-time reads).}
        With $\bits(\terminal{mem}[0])$ in the residual stream, a constant-depth MLP computes the match bit $\delta_p \defeq \indFun{\bits(\terminal{mem}[0]) = \bits(p)}$ for every program address $p\in\{0,\ldots,P-1\}$.
        The weights of $T$ contain $P$ head groups $\mathcal{H}_0,\ldots,\mathcal{H}_{P-1}$, one per program line; each $\mathcal{H}_p$ issues, in parallel, one indexed read (\Cref{lem:reading-from-memory}) for every compile-time address in $I_p$'s descriptor.
        All $P$ groups run as parallel heads within the same two attention layers; an MLP gated by $\delta_p$ then copies the active group's results into shared operand slots.

        \emph{Sub-step 2b (runtime reads).}
        For each runtime address listed in $I_p$'s descriptor (the value $\bits(a)$ for a $\terminal{mem}[a]$ or $\terminal{inp}[a]$ access), $\bits(a)$ is now available in the residual stream from Sub-step~2a.
        A second indexed read uses $\bits(a)$ as the query: \Cref{lem:reading-from-memory} for $\terminal{mem}[a]$, or a direct positional attention to input position $a$ for $\terminal{inp}[a]$.
        By the \emph{flat} \wRAM{} constraint, no runtime address is itself memory-indirect, so two more attention layers suffice.
        After Block~2 every operand of $I_p$ resides in its designated residual-stream slot.

        \item[\textbf{Block 3: Execute Instruction} (\Cref{prop:logwidth-execute}).]
        With all operands in the residual stream, a single constant-depth MLP computes the $\wordsize$-bit destination value $\bits(v)$ and the next \texttt{pc} value $\bits(\terminal{pc}')$, placing both into dedicated result slots.
        For instructions whose result is a $\wordsize$-bit function of $\wordsize$-bit operands (bitwise, $+$, $-$, $\ll$, $\gg$, comparisons, conditional moves, branch-target arithmetic), Block~3 implements the operation as an $\mathrm{AC}^0$ circuit and uses no attention.
        For $\times$, $\div$, $\bmod$, Block~3 executes \emph{one inner iteration} of the corresponding long-multiplication or long-division algorithm; the per-iteration intermediates (partial product, remainder, quotient, phase counter) live in program-private memory cells, so they are read and written like any other word and the simulator-level treatment is unchanged (\Cref{prop:logwidth-mul-div}).

        \item[\textbf{Block 4: Store Results} (\Cref{prop:logwidth-store}).]
        Block~4 appends, to the transcript, the memory blocks committed by the active instruction: a $\memkey\,\BOV\,\bits(a_d)\,\BOV\,\bits(v)$ block for each address $a_d$ that $I_p$ writes (typically the destination plus the updated \texttt{pc}), followed by a closing \EOI{}.
        The transformer's weights are shared across positions, so Blocks~1--3 \emph{run} at every position; their reads are reliable, however, only at \EOI{} positions, where the rightmost-match reads see exclusively committed memory blocks from \emph{prior} steps.
        At any non-\EOI{} position within the current step's serialization, a rightmost-match read for the destination address $a_d$ would retrieve the freshly (and partially) written value of the current step rather than its operand-time value.
        Block~4 sidesteps this by ignoring the local outputs of Blocks~1--3 at non-\EOI{} positions and instead attending to the most recent \EOI{}, copying the snapshot (the addresses to be written and the corresponding result bits) that Blocks~1--3 produced at that \EOI{}.
        Concretely: at the \EOI{} position itself, Block~4 emits the first token of the next step's serialization (the $\memkey$ token of the first memory block); at each of the following $C\wordsize + D - 1$ positions, an attention head attends to the rightmost \EOI{} at position $s_\EOI$, an MLP computes the offset $\delta \defeq \cotstep - s_\EOI \in \{1,\ldots,C\wordsize + D\}$, and a constant-depth MLP selects the appropriate $\delta^{\text{th}}$ token of the serialization (a $\memkey$ marker, a \BOV{} marker, an address bit, a value bit, or the closing \EOI{}) from the snapshot via one-hot extraction \citep[][Lem.~B.1]{svete-cotterell-2024-transformers}.
        When the offset hits a block boundary, the next memory block's $\memkey$ and \BOV{} tokens are emitted; when all blocks have been written, the closing \EOI{} is emitted, at which point Blocks~1--3 will again run on a clean prefix and produce the snapshot for the next step.
    \end{description}

    \textbf{Token count.}
    Each \wRAM{} step writes a constant number $K$ of memory blocks (depending only on the program, not the input), so
    \[
        C\wordsize + D \;=\; 2K\wordsize + 3K + 1
    \]
    tokens are emitted per step (each block is $\memkey$, two \BOV{}s, and $2\wordsize$ bit tokens; one closing \EOI{}).
    For instructions whose semantics is a single $\mathrm{AC}^0$ operation (simple-arithmetic class) this is $\bigOFun{\wordsize}$ per step.
    For $\times$/$\div$/$\bmod$, the operation is unfolded into $\wordsize$ inner iterations, each of which is itself a flat memory-read/write step taking $\bigOFun{\wordsize}$ tokens (\Cref{prop:logwidth-mul-div}); the total is therefore $\bigOFun{\wordsize^2}$.
    Summing over $t$ \wRAM{} steps gives $t\cdot\bigOFun{\wordsize^2}$ CoT tokens, and $t\cdot\bigOFun{\wordsize}$ when no $\times$/$\div$/$\bmod$ instruction is used.

    \textbf{Residual stream as a serialized snapshot.}
    It is instructive to view the residual stream at any CoT position as a structured record with one $\wordsize$-dimensional slot per memory block to be emitted by the current step: an address slot and a value slot for each of the $K$ writes.
    Each CoT token fills one bit of one slot, advancing $\delta$ through the slots until \EOI{} marks completion.
    The CoT transcript is therefore a sequence of $\memkey$-anchored memory blocks, and the simulator's job at each step reduces to (i) reading a constant number of $\memkey$ blocks, (ii) computing the new contents, and (iii) appending a constant number of new $\memkey$ blocks.
    \Cref{fig:logwidth-layer-blocks} (right) illustrates this view alongside the transcript structure.
\end{proof}

\begin{proposition}[Load \texttt{pc}]\label{prop:logwidth-load-pc}
    Let $T$ be a rightmost \uhat{} with $\log^2$ width and fixed precision, and suppose the serialization format above is in use.
    At every CoT position $\cotstep$, $T$ can recover $\bits(\terminal{mem}[0])$ from the transcript and place it in a dedicated residual-stream subspace using two attention layers.
\end{proposition}

\begin{proof}
    The query address is the constant $0$, hard-coded into the transformer weights and available in the residual stream as $\bits(0)$.
    A single application of \Cref{lem:reading-from-memory} with this query retrieves $\bits(\terminal{mem}[0])$:
    the first attention layer matches the rightmost $\memkey$ block whose address field encodes $0$ and extracts the position of the matching block; the second attention layer attends to the value field of that block and assembles its bits, via \Cref{lem:log-width-serializing-deserializing}, into the dedicated $\bits(\terminal{pc})$ subspace.
    If no such block exists, the override of the previous paragraph produces the program-supplied initial \texttt{pc} value $\bits(\PCstart)$.
\end{proof}

\begin{proposition}[Load instruction and operands]\label{prop:logwidth-dispatch}
    Let $T$ be a rightmost \uhat{} with $\log^2$ width and fixed precision, and suppose $\bits(\terminal{mem}[0])$ is in the residual stream (from Block~1).
    Then $T$ can identify the active instruction and load all its operands into designated residual-stream slots in four attention layers.
\end{proposition}

\begin{proof}
    \textbf{Operand descriptor.}
    Every flat \wRAM{} instruction $I_p$ is equipped with a static \emph{operand descriptor}: a list of memory addresses to read.
    The descriptor partitions the addresses into
    \begin{itemize}[noitemsep]
        \item \emph{compile-time addresses} $A^{\mathrm{ct}}_p \subseteq \wordSet$: integer constants known at program-compile time (the registers $-i$, the \texttt{pc} address $0$, and any program-private scratch addresses used by $I_p$);
        \item \emph{runtime addresses} $A^{\mathrm{rt}}_p$: addresses of the form $\terminal{mem}[a]$ or $\terminal{inp}[a]$ where $a$ itself is the value at some compile-time address.
    \end{itemize}
    The descriptor is a \emph{compile-time constant}: it depends only on the instruction type, not on the runtime values, so it is hard-coded into the weights of $T$.

    \textbf{Sub-step 2a: Instruction identification and compile-time reads.}
    With $\bits(\terminal{mem}[0])$ in the residual stream, a constant-depth MLP computes the match bit
    \[
        \delta_p \;\defeq \; \indFun{\bits(\terminal{mem}[0]) = \bits(p)} \;\in\; \{0,1\}
    \]
    for every program address $p \in \{0,\ldots,P-1\}$ simultaneously (each $\bits(p)$ is a hard-coded weight constant).
    The weights of $T$ contain $P$ attention-head groups $\mathcal{H}_0, \ldots, \mathcal{H}_{P-1}$, one per program line.
    All $P$ groups run in parallel: each group $\mathcal{H}_p$ issues, for every compile-time address $a \in A^{\mathrm{ct}}_p$, one indexed read via \Cref{lem:reading-from-memory} with the constant query $\bits(a)$, writing the retrieved $\bits(\terminal{mem}[a])$ into a dedicated disjoint subspace of the residual stream.
    Each indexed read uses two attention layers (address matching, then value retrieval), and all $P$ groups run as parallel heads within those same two layers.
    A subsequent MLP gated by $\delta_p$ copies the active group's operands from their dedicated subspace into the shared operand slots, leaving the other subspaces unused.
    After Sub-step~2a, every compile-time operand of $I_p$ occupies its designated slot.

    \textbf{Sub-step 2b: Runtime reads.}
    For each runtime address in $A^{\mathrm{rt}}_p$, the address value $\bits(a)$ resolved in Sub-step~2a is now in the residual stream and serves as the query for a second indexed read: \Cref{lem:reading-from-memory} for $\terminal{mem}[a]$, or a direct positional attention to input position $a$ for $\terminal{inp}[a]$.
    By the \emph{flat} \wRAM{} constraint, no runtime address is itself memory-indirect, so two more attention layers suffice.
    After Sub-step~2b, all operands required by $I_p$ reside in their designated residual-stream slots and are ready for Block~3.

    The total depth of Block~2 is four attention layers plus a constant number of MLP layers.
\end{proof}

\begin{proposition}[Execute instruction]\label{prop:logwidth-execute}
    Suppose all operands of instruction $I_p$ reside in their designated residual-stream slots (from Block~2).
    Then $T$ can compute, in constant depth using a single MLP with no attention, the $\wordsize$-bit value $\bits(v_d)$ to be written to each address $a_d$ that $I_p$ updates, including the next \texttt{pc} value $\bits(\terminal{pc}')$.
\end{proposition}

\begin{proof}
    With all operands $\bits(a), \bits(b) \in \{0,1\}^\wordsize$ in the residual stream, the active gated MLP subnetwork $\mathcal{N}_p$ computes the destination value(s) $\bits(v_d)$ directly:
    \begin{itemize}
        \item \emph{Bitwise operators} ($\mathbin{\&}$, $\mathbin{|}$, $\oplus$, $\ll$, $\gg$): each $v_t$ is an $O(1)$-depth Boolean function of $a_t$, $b_t$ (and a hard-coded shift amount), computed in one MLP layer.
        \item \emph{Addition and subtraction}: implementable by a constant-depth $\text{AC}^0$ circuit \citep[][\S VIII.1]{straubing}, and thus by an MLP.
        \item \emph{Comparison} ($<, \le, =, \neq, \ge, >$): reduced to subtraction and sign extraction.
        \item \emph{Boolean and conditional} operations: computed in $O(1)$ depth from the Boolean bits in the residual stream.
        \item \emph{Move and branch targets}: the result is a hard-coded constant or a copy of an operand value already in the residual stream.
        \item \emph{Inner iteration of $\times$/$\div$/$\bmod$}: the MLP executes one iteration of the long-multiplication or long-division algorithm using the per-iteration intermediates (partial product, remainder, quotient, phase counter) read in Block~2 from program-private memory addresses, and produces their updated values as additional destination words (\Cref{prop:logwidth-mul-div}).
    \end{itemize}
    Block~3 always computes, in addition, the next \texttt{pc} value $\bits(\terminal{pc}') = \bits(\terminal{mem}[0]) + \bits(1)$ (or a hard-coded branch target if $I_p$ is a branch); this value is written to address $0$ by Block~4.
    Block~3 uses no attention, and the destination values remain stable at every CoT position of the same step since the operands are fixed and Blocks~1--2 recover the same values each time.
\end{proof}

\begin{proposition}[Store results]\label{prop:logwidth-store}
    Suppose that at every \EOI{} position $s_{\EOI}$, the full snapshot for the next step---the set of destination addresses $\{a_d\}$ together with $\bits(a_d)$ and the corresponding result words $\bits(v_d)$---is in the residual stream (from Blocks~1--3 evaluated at $s_{\EOI}$).
    Then $T$ can, at every position $\cotstep > s_{\EOI}$ within the step, attend to the rightmost \EOI{}, compute the offset $\delta \defeq \cotstep - s_{\EOI}$, and emit the correct $\delta^{\text{th}}$ token of the step's serialization (a $\memkey$ marker, a \BOV{} marker, an address bit, a value bit, or the closing \EOI{}) in constant depth, using one attention layer and one MLP.
\end{proposition}

\begin{proof}
    \textbf{Anchoring at the most recent \EOI{}.}
    By construction, \EOI{} tokens partition the transcript into steps: each step's serialization is a run of $C\wordsize + D$ tokens following an \EOI{} (and ends with the next \EOI{}), where $C = 2K$ and $D = 3K + 1$ for $K$ the (program-determined) number of memory blocks committed by the active step.
    An attention head issues a query matching the \EOI{} token type via \Cref{lem:logwidth-attention}, uniquely attending to the rightmost \EOI{} at position $s_{\EOI} < \cotstep$ (or to the input boundary token at the very first step).
    The retrieved value vector is the positional encoding $\posEnc(s_{\EOI}) = \bits(s_{\EOI})$.
    A subsequent MLP computes the within-step offset $\delta \defeq \cotstep - s_{\EOI}$ from $\posEnc(\cotstep)$ and $\bits(s_{\EOI})$; this difference ranges over $\{1,\ldots,C\wordsize+D\}$.
    The same attention head also retrieves, from the residual stream at $s_{\EOI}$, the snapshot (the destination addresses $\bits(a_d)$ and result values $\bits(v_d)$ for every memory block of the current step, in fixed order) that Blocks~1--3 produced at $s_{\EOI}$.

    \textbf{Selecting and emitting the $\delta^{\text{th}}$ token.}
    The serialization layout of a step is fixed once dispatched: a sequence of memory blocks of the form $\memkey\,\BOV\,\langle\wordsize\text{ address bits}\rangle\,\BOV\,\langle\wordsize\text{ value bits}\rangle$ at known offsets, followed by a closing \EOI{}.
    The MLP converts $\delta$ to its one-hot encoding $\ve_\delta \in \{0,1\}^{C\wordsize + D}$ and uses it to select among:
    \begin{enumerate}[label=\textup{(\arabic*)},noitemsep]
        \item the $\memkey$ token when $\delta$ is the offset of the start of a memory block;
        \item a \BOV{} marker when $\delta$ is the offset of an address-field or value-field start;
        \item an address bit $b_j = \inner{\bits(a_{\mathrm{block}(\delta)})}{\ve_{j(\delta)}}$ \citep[][Lem.~B.1]{svete-cotterell-2024-transformers} when $\delta$ falls inside an address field;
        \item a value bit $b_j = \inner{\bits(v_{\mathrm{block}(\delta)})}{\ve_{j(\delta)}}$ when $\delta$ falls inside a value field, where $\mathrm{block}(\delta)$ and $j(\delta) \in \{1,\ldots,\wordsize\}$ are the block index and intra-field bit index determined from $\delta$ alone;
        \item the closing \EOI{} when $\delta = C\wordsize + D$, which simultaneously triggers the next step's full Block~1--3 evaluation at that position.
    \end{enumerate}
    The selected token is emitted as the output at position $\cotstep$.

    \textbf{Why the snapshot is reliable at \EOI{} positions.}
    Blocks~1--3 are run at every position $\cotstep$ uniformly; what changes across positions is the prefix they see.
    At an \EOI{} position $s_{\EOI}$, all rightmost-address reads see the previous step's committed memory blocks (no half-written block of the new step has yet been emitted), so the operand reads are unambiguous and Blocks~1--3 produce the correct snapshot for the next step.
    At any non-\EOI{} position within the step's serialization, the same machinery runs but its rightmost-match reads now collide with the partially emitted memory block of the current step: in particular, a read for any destination address $a_d$ would retrieve the (possibly partial) freshly written value rather than the operand-time value.
    Block~4 sidesteps this by anchoring on the rightmost \EOI{} and copying the snapshot computed there, where only fully committed memory blocks are visible; the local outputs of Blocks~1--3 at non-\EOI{} positions are disregarded.
\end{proof}

\begin{proposition}[Multiplication, division, and modulo via per-iteration memory cells]\label{prop:logwidth-mul-div}
    Let $I$ be a $\times$, $\div$, or $\bmod$ instruction with operands $a$ and $b$.
    Then $T$ can compute the result of $I$ in $\bigO(\wordsize)$ \wRAM{} steps, each costing $\bigO(\wordsize)$ CoT tokens, for a total of $\bigO(\wordsize^2)$ CoT tokens, using only the four-block pipeline above.
\end{proposition}

\begin{proof}
    \textbf{Why $\bigO(\wordsize^2)$ tokens are necessary.}
    The standard bit-serial long-multiplication and long-division algorithms run $\wordsize$ outer iterations, each updating an $\wordsize$-bit intermediate.
    Compiling the instruction into $\wordsize$ flat \wRAM{} micro-steps that share state through memory yields a $\wordsize \times \bigO(\wordsize) = \bigO(\wordsize^2)$ token budget; each micro-step is exactly the standard four-block pipeline, with no special control flow.

    \textbf{Per-iteration state in memory.}
    The program reserves four private memory addresses $\alpha_P$, $\alpha_R$, $\alpha_q$, $\alpha_\pi$ (constants, fixed at compile time) for the partial product $P$, the running remainder $R$, the partial quotient $q$, and a phase counter $\pi$, respectively, plus two more $\alpha_a, \alpha_b$ for the operand snapshot.
    The compiled implementation of $I$ is a program fragment that
    \begin{enumerate}[label=\textup{(\arabic*)},noitemsep]
        \item on entry, writes the operands $a, b$ to $\alpha_a, \alpha_b$, zero-initializes $P, R, q$, and sets $\pi \leftarrow 1$;
        \item executes a loop of $\wordsize$ flat micro-steps, each of which reads $a, b$, the current intermediates, and $\pi$, performs one algorithm iteration, writes the updated intermediates and $\pi \leftarrow \pi + 1$, and branches on $\pi$;
        \item on exit, writes the final result to the destination address.
    \end{enumerate}
    Each micro-step has compile-time-fixed read and write address sets (entries of $\{\alpha_a, \alpha_b, \alpha_P, \alpha_R, \alpha_q, \alpha_\pi\}$), so the operand descriptors and four-block pipeline of \Cref{prop:logwidth-load-pc,prop:logwidth-dispatch,prop:logwidth-execute,prop:logwidth-store} apply unchanged: each micro-step costs $\bigO(\wordsize)$ CoT tokens.

    \textbf{Multiplication ($a \times b \bmod 2^\wordsize$).}
    Indexing bits LSB-first, micro-step $\pi$ extracts $b_\pi = \inner{\bits(b)}{\ve_\pi}$ in Block~3 and computes $P' = P + a \cdot b_\pi \cdot 2^\pi \pmod{2^\wordsize}$.
    After $\wordsize$ micro-steps, $P = a \times b \pmod{2^\wordsize}$, and a final step copies $P$ to the destination.

    \textbf{Division and modulo ($a \div b$, $a \bmod b$).}
    Indexing bits LSB-first, micro-step $\pi$ processes bit $i = \wordsize - \pi$ of $a$ (MSB to LSB).
    Block~3 extracts $a_i = \inner{\bits(a)}{\ve_i}$, computes $R' = (R \ll 1) \mathbin{|} a_i$, performs the trial subtraction $R'' = R' - b$, and tests the sign via the high bit of the $(\wordsize{+}1)$-bit result.
    If $R'' \geq 0$: $q' = q \mathbin{|} (1 \ll i)$, $R \leftarrow R''$; otherwise: $q' = q$, $R \leftarrow R'$.
    After $\wordsize$ micro-steps, $q = \lfloor a/b \rfloor$ and $R = a \bmod b$, and a final step copies the appropriate value to the destination.

    Both algorithms use $\bigO(\wordsize^2)$ CoT tokens via the standard four-block pipeline.
\end{proof}

\begin{restatable}[Indexed read for $\log^2$-width transformers]{lemma}{indexedRead}\label{lem:indexed-read-logwidth}
    Let $T$ be a fixed-precision, $\log^2$-width \uhat with positional encodings $\posEnc(j) = \bits(j)$.
    Let $k \in \wordSet$ be a query key available as $\bits(k)$ in the residual stream at position $\inputidx$.
    Suppose the CoT transcript contains a sequence of key--value blocks: at positions $s_1 < s_2 < \cdots < s_m \leq \inputidx - 1$, each block $s_i = (k_i, v_i)$ starts with a key field of $\wordsize$ consecutive bit tokens $\bits(k_i) \in \set{0,1}^{\wordsize}$ followed by a value field of $\wordsize$ consecutive bit tokens $\bits(v_i) \in \set{0, 1}^\wordsize$.
    Let $s^*$ be the start of the rightmost block whose key field encodes $k$; $k^* = k$.
    Then we can augment $T$ with a constant-depth subnetwork that, at position $\inputidx$, recovers
    \[
    \bits(v^*) = \sum_{t=0}^{\wordsize-1} b_t \ve_t
    \]
    in the residual stream after the first transformer layer, where $\bits(v^*) = \begin{pmatrix}
        b_0 & \cdots & b_{\wordsize-1}
    \end{pmatrix}$, with the default value $\vzero$ when no matching block exists.
\end{restatable}
\begin{proof}
    \textbf{Step 1: Key matching.}
    Using an affine transformation of $\bitsFun{k_i}$ as the key vectors and $\bitsFun{k}$ as the query vector, the attention scores among positions $1, \ldots, \inputidx$ are maximized at $s_j$ for $j = 1, \ldots, m$ (\cref{lem:logwidth-attention}).
    Rightmost \uhat selects the position $s^*$, extracts its positional encoding as the value vector $\vv_{s^*}$, and adds to it the constant offset $\wordsize$ with an MLP.
    
    \textbf{Step 2: Value retrieval.}
    Position $\inputidx$ now has access to the position of the value $v^*$.
    The next attention layer can uniquely attend to this position and retrieve the value (\cref{lem:logwidth-attention}).
\end{proof}

\begin{lemma}[Serializing-Deserializing values via CoT] \label{lem:log-width-serializing-deserializing}
    Let $T$ be a fixed-precision, $\log^2$-width \uhat{}.
    Suppose that at position $\inputidx$ the output symbol is $\BOV{}$ and the residual stream contains $\bits(v) = \begin{pmatrix} b_1 & \cdots & b_\wordsize \end{pmatrix} \in \set{0,1}^{\wordsize}$.
    Then, $T$ can be augmented with constant-depth subnetworks such that:
    \begin{enumerate}[label=(\arabic*)]
        \item \textbf{(Write)} With $\wordsize$ CoT steps, $T$ can output $b_t$ in position $\inputidx+t$ for all $t \in \{1, \ldots, \wordsize\}$.
        \item \textbf{(Read)} $T$ has $\bits(v)$ available in the residual stream at position $\inputidx+\wordsize$.
    \end{enumerate}
\end{lemma}

\begin{proof}
    \textbf{Part (1): Write.}
    $T$ at position $\inputidx + t$ must emit $b_t$, the $t\textsuperscript{th}$ bit of $v$.
    By hypothesis, $\bits(v)$ is available in the residual stream at position $\inputidx + t$ (recomputed by the same constant-depth subnetwork at each instruction-mode output step; cf. \Cref{prop:logwidth-execute}).
    To identify $t$, we use positional encodings $\posEnc(j) = \bits(j)$.
    At position $\inputidx + t$, an attention layer uses rightmost-unique hard attention to attend to the last preceding $\BOV{}$ symbol, which occupies position $\inputidx$ (cf. \Cref{lem:indexed-read-logwidth}).
    The attending position extracts $\bits(\inputidx)$ from the value at position $\inputidx$.
    A subsequent MLP subtracts $\bits(\inputidx)$ from the local positional encoding $\bits(\inputidx + t)$, yielding $\bits(t)$.
    A second MLP layer then converts $\bits(t)$ into the one-hot vector $\ve_t \in \{0,1\}^\wordsize$, which can be done as a conjunction of the bits of $\bits(t)$ \citep[][Lemma B.1]{svete-cotterell-2024-transformers}.
    Finally, an MLP computes $\inner{\bits(v)}{\ve_t} = b_t$, which is emitted as the output token at position $\inputidx + t$.
    
    \textbf{Part (2): Read.}
    At position $\inputidx + \wordsize$, we must reconstruct $\bits(v) = \sum_{t=0}^{\wordsize-1} b_t \ve_t$ in the residual stream, where $\ve_t$ is the $t\textsuperscript{th}$ standard basis vector.
    We use $\wordsize$ dedicated hard-attention heads $H_0, H_1, \dots, H_{\wordsize-1}$, one per bit.
    Head $H_t$ must attend to position $\inputidx + t$, i.e., to the earlier position that emitted bit $b_t$.
    
    Head $H_t$ at position $\inputidx + \wordsize$ computes a query from $\bits(\inputidx + \wordsize) = \posEnc(\inputidx + \wordsize)$ and keys from $\bits(j) = \posEnc(j)$ at each earlier position $j$.
$H_t$ can identify position $\inputidx + t$ exactly by forming the query $\posEnc(\inputidx + \wordsize) - \bits(\wordsize - t) = \posEnc(\inputidx + t)$, and matching it against the key $\posEnc(j)$.
    Hard attention then selects position $\inputidx + t$ holding the symbol $\inputsym{\inputidx + t}$ as the unique position with key equal to the query.
    
    The value map of head $H_t$ returns $\inputsym{\inputidx + t} \cdot \ve_t = b_t \cdot \ve_t$.
    After summing over all $\wordsize$ heads, the first transformer layer at position $\inputidx + \wordsize$ produces
\[
    \sum_{t=0}^{\wordsize-1} \inputsym{\inputidx+t} \cdot \ve_t = \sum_{t=0}^{\wordsize-1} b_t \cdot \ve_t = \bits(v),
\]
    which is $\repr(v)$ as required.
\end{proof}

\begin{remark}[Why $\log^2$ width is necessary]\label{rem:log-squared-width}
    The Read step above uses $\wordsize = \bigTheta(\log \inputlen)$ heads, each of width $\bigTheta(\log \inputlen)$, for a total model dimension of $\bigTheta(\log^2 \inputlen)$.
    Both factors are forced.
    
    \textit{Logarithmically many heads.}
    To reconstruct $\bits(v)$ at a single position, one must gather $\wordsize$ bits from $\wordsize$ distinct earlier positions.
    Compressing information from a log-growing number of positions into fewer constantly-many heads is infeasible at constant precision: the attention weight at each position is a scalar, so with $o(\wordsize)$ heads the aggregated output is a fixed-dimensional summary of $o(\wordsize)$ positions, which cannot faithfully recover all $\wordsize$ bits.
    One could in principle design highly bespoke positional encodings (with L-uniform PEs; cf. \citet{london2025pausetokensstrictlyincrease} and \citet{svete2026on}) that pack multiple bits into a single head's key-query score, but this would substantially complicate the construction and break the clean one-bit-per-head structure; we therefore use $\bigTheta(\wordsize)$ heads.
    
    \textit{Logarithmic width per head.}
    Each head $H_t$ must attend to a specific position $\inputidx + t$, which it identifies by forming the query $\posEnc(\inputidx + \wordsize) - \bits(\wordsize - t) = \posEnc(\inputidx + t)$ and matching it against position keys.
    Carrying the positional offset $\bits(\wordsize - t)$ in the query already requires $\theta(1)$ bits of information per head; a constant-width head cannot encode this offset.
    Hence each head must itself be of growing width.
    
    Together, $\bigTheta(\wordsize)$ heads of width $\bigTheta(\wordsize)$ give total dimension $\bigTheta(\wordsize^2) = \bigTheta(\log^2 \inputlen)$.
    Note that the structure is highly uniform: all $\wordsize$ heads are simple affine transformations of the same $\log \inputlen$-wide positional encoding $\posEnc$, differing only in a constant additive offset $\bits(\wordsize - t)$.
\end{remark}

\begin{remark}[Why rightmost \uhat{} is necessary]\label{rem:rightmost-necessary}
    The indexed-read step requires attending to the \emph{rightmost} position in the transcript whose key field matches a query address $a \in \{0,1\}^\wordsize$.
    We sketch why constant-precision \ahat{} (and equivalently leftmost \uhat{} \citep{jerad-etal-2025-unique,li2026characterizing}) cannot implement this.
    
    Abstractly, the attention score at position $j$ decomposes as
    \[
    \score(\inputidx, j) = \inner{\vq_\inputidx}{\vk_j} = \underbrace{\bits(a)^\top \bits(a_j)}_{\text{key match}} + \underbrace{m(j)}_{\text{rightmost bias}},
    \]
    where $a_j$ is the key address at position $j$ and $m \colon \N \to \R$ must be strictly increasing so that the rightmost matching position wins.
    The key-match term equals $\bigThetaFun{\wordsize}$ when $a_j = a$ and is strictly less otherwise, so $m$ must simultaneously satisfy:
    \begin{enumerate*}[label=\textit{(\roman*)}]
        \item strict monotonicity $m(j+1) > m(j)$, which forces $m(j + 1) - m(j) = \bigOmega(1)$, thus $m(j) = \bigOmega(j) = \bigOmega(\inputlen)$ over a transcript of length $\inputlen$ due to constant precision; and
        \item a dominance bound $m(j) < \bigOFun{\log \inputlen}$ so that a mismatched position cannot outscore a matched one.
    \end{enumerate*}
    These two requirements are contradictory, so rightmost \uhat{} is necessary.\footnote{Greater expressivity of rightmost \uhat{} has been established independently in the fixed-precision and fixed-width settings \citep{yang2024masked,jerad-etal-2025-unique,li2026characterizing}; the argument above shows the same holds in the fixed-precision growing-width setting.}
    
    By contrast, attending to the \emph{leftmost} match is easy under \ahat{}: one only needs a boolean flag indicating whether the same address appeared earlier, which is computable in constant precision.
    Logarithmic precision also circumvents the barrier by encoding offsets $m$ with $m(j) < 1$ for all $j$ (circumventing the requirement that $m(j + 1) - m(j) = \bigOmega(1)$); existing transformer Turing completeness results rely on exactly that \citep{perez2021attention,nowak-etal-2024-representational}.
\end{remark}

\begin{lemma}[Reading from Memory]\label{lem:reading-from-memory}
    Let $T$ be a rightmost \uhat with $\log^2$ width and fixed precision. Suppose the CoT transcript contains a store log where each entry is serialized as $\memkey \, \BOV \, \bits(a_j) \,\BOV \, \bits(v_j)$. Given an address $a \in \wordSet$ available as $\bits(a)$ in the residual stream at position $\inputidx$, $T$ can retrieve the most recent value $v$ stored at address $a$ in constant depth.
\end{lemma}
\begin{proof}
    The retrieval is performed in two attention steps:
    
    \textbf{Step 1: Address Matching.}
    $T$ generates a query $\bitsFun{a}$ and attends to all positions $j < \inputidx$ following an $\memkey$ marker (these can be detected by looking back a constant number of positions).
    By \Cref{lem:indexed-read-logwidth}, the rightmost \uhat mechanism selects the starting position $s^*$ of the most recent block where the $\wordsize$ bits following $\memkey$ perfectly match $\bits(a)$.
    The Transformer then extracts the positional encoding $\posEnc(s^*)$ into the residual stream.
    
    \textbf{Step 2: Value Retrieval.}
    Using the retrieved $\posEnc(s^*)$, $T$ computes the position of the corresponding value field by adding to it a constant offset.
    A second attention layer uniquely attends to this position and uses the bit-retrieval mechanism of \Cref{lem:log-width-serializing-deserializing} to reconstruct $\bits(v) = \sum_{t=0}^{\wordsize-1} b_t \ve_t$ in the residual stream.
    
    If no matching $\memkey$ is found, the attention mechanism defaults to a null vector, resulting in $\bits(0)$ as per the simulation invariant.
\end{proof}

The following is a restatement of \citet[][Lem.\ D.7]{svete2026on}, used throughout this section.
\begin{lemma}[\citealt{svete2026on}, Lem.\ D.7] \label{lem:logwidth-attention}
    Let $T$ be a fixed-precision polylogarithmic-width transformer such that its residual stream contains $\bits(v)$ at position $\inputidx$.
    Then, $T$'s attention heads can attend to all positions $j \leq \inputidx$ that contain $\bits(v)$ in the residual stream as well (and put zero attention to others).
    If $T$ is a rightmost \uhat{}, the rightmost such position is uniquely attended to.
\end{lemma}

\section{Proofs of Continuous CoT Results} \label{app:ccot}

%
%
%
\begin{figure}[t]
\centering
\input{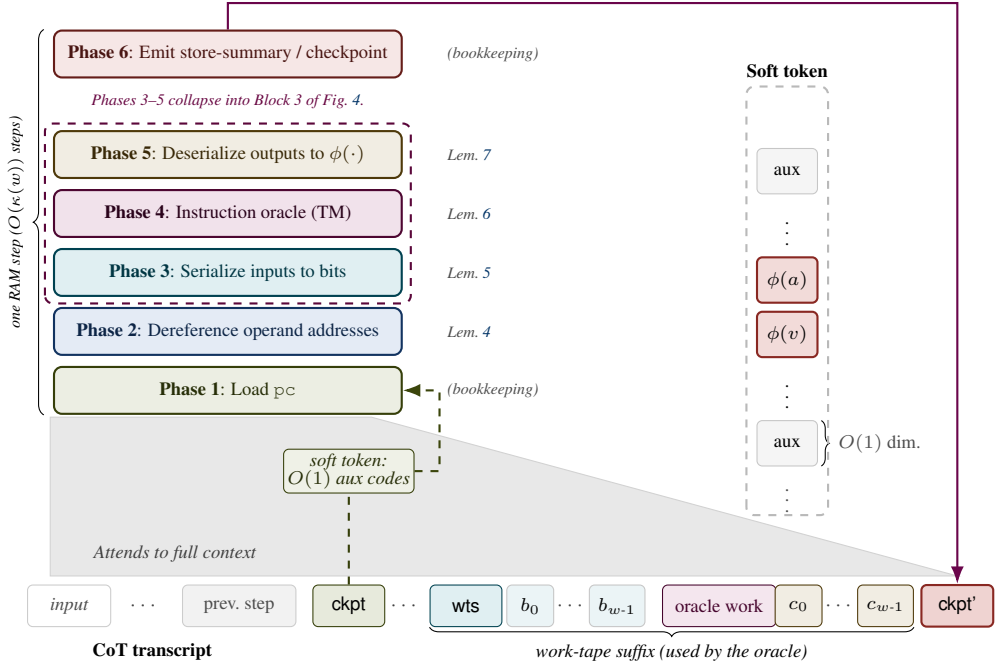}
\caption{\textbf{Phase structure and CoT transcript view} of one RAM step in the CCoT (continuous CoT) simulation; counterpart to \Cref{fig:logwidth-layer-blocks}.
    \emph{Phases (center):} unlike the $\log^2$-width construction in which one CoT step executes four constant-depth layer blocks, in the CCoT simulation a single RAM step expands across $O(\instrCostFun{\wordsize})$ decoding steps grouped into six conceptual phases (cf.\ the proof overview of \Cref{thm:wordram-ccot}).
    Phase~1 forms $\phi(0)$ and reads the current $\phi(\PC) = \phi(\terminal{mem}[0])$ from the rightmost matching store-summary, selecting the active instruction (under the unified memory convention, the program counter and registers all live at fixed addresses in $\terminal{mem}$, so reads use a single dereferencing primitive).
    Phase~2 dereferences each memory operand $\terminal{mem}[a]$ by forming $\phi(a)$ and retrieving $\phi(v)$ from the rightmost matching store-summary (\Cref{lem:ccot-dereferencing}).
    Phase~3 emits the $w$ bits of each operand onto the transcript while keeping its $\phi(\cdot)$ code in the soft token (\Cref{lem:ccot-serialization}).
    Phase~4 runs the instruction-specific TM oracle on the work-tape suffix, which streams output bits over $O(\instrCostFun{\wordsize})$ decoding steps (\Cref{lem:ccot-instruction-execution}).
    Phase~5 converts each output bit stream back to $\phi(\cdot)$ for future addressing/equality (\Cref{lem:ccot-deserialization}).
    Phase~6 emits a store-summary token for each write the instruction performs to $\terminal{mem}$ (including writes to address $0$ for the new \terminal{pc} and to addresses $-1,-2,\ldots$ for any updated registers), followed by an end-of-step checkpoint, re-establishing the transcript invariant.
    The dashed purple rectangle marks Phases~3--5, which together implement the bit-serial execution of the instruction; in the fixed-precision (polylog-width) construction these three phases collapse into a single constant-depth MLP block (Block~3 of \Cref{fig:logwidth-layer-blocks}), since values can be manipulated directly as $\wordsize$-bit residual-stream slots without the serialize/oracle/deserialize round trip needed to expose them as $w$ discrete tokens.
    \emph{CoT transcript (bottom):} the prev.\ checkpoint (green) carries the cached state in its soft token; serialized operand bits, the oracle's work region, and the deserialized output bits form the work-tape suffix; the rightmost \textsf{ckpt'} (red) is the new checkpoint emitted at the end of this step.
    \emph{Soft token (right):} the residual stream carries a constant number of $O(1)$-dimensional codes $\phi(\cdot)$, in contrast to the $\wordsize$-dimensional bit slots of \Cref{fig:logwidth-layer-blocks}. Under the unified memory convention, no slot is reserved for a specific register: highlighted slots $\phi(a),\phi(v)$ (red) hold the address being looked up and the value being loaded/written at the current step, and remaining slots are generic $O(1)$ scratch (\emph{aux}) used inside bit-serial routines.}
\label{fig:ccot-phase-blocks}
\end{figure}

\serializationLemmaCCoT*

\begin{proof}
    Write $v=\sum_{k=0}^{w-1} b_k2^k$ and define $v_0:=v$, $v_{k+1}:=\lfloor v_k/2\rfloor$, and $b_k:=v_k\bmod 2$.
    
    At decoding step $k$ (position $\tau+k+1$), we have $v_k\le v\le \tau < \tau+k+1$, hence $v_k\le \tau+k+1$.
    Therefore we may apply Lemma~\ref{lem:division} with $(v,d)=(v_k,2)$ to compute $\phi(\lfloor v_k/2\rfloor)=\phi(v_{k+1})$ and $\phi(v_k\bmod 2)=\phi(b_k)$ in constant depth.
    A constant-depth equality test against the constants $\phi(0)$ and $\phi(1)$ selects the token $\mathtt{0}/\mathtt{1}$ to be emitted, along with $\phi(v_{k+1})$ as a soft token to be used at the next position.
    After $w$ steps the emitted tokens are exactly $b_0,\dots,b_{w-1}$, along with soft tokens for intermediate $\phi(v_k)$ values.
\end{proof}

\deserializationLemmaCCoT*

\begin{proof}
    As we read the bits of $v$ from least significant to most, we can compute $v / 
    2^w$ by keeping a running sum that starts at $0$ and is updated as follows: at each step, cut it in half and add the next bit. Concretely, let $v_0 = 0$ and, for $1 \leq \ell \leq w$, define $v_\ell = \frac{1}{2} v_{\ell-1} + b_{\ell-1}$. It can be verified that $v_w = v / 2^w$.
    
    To perform this computation with CoT using soft tokens, we start with $v_0 = 0$, represented internally as $\phi(0)$. Thus we have access to $\phi(v_0)$ when bit $b_0$ is ready to be serialized. By induction, we assume that when bit $b_{\ell-1}$ is serialized using a soft token, we also have $\phi(v_{\ell-1})$ available and can serialize it as part of the same soft token. In the next step, we read both $\phi(b_{\ell-1})$ and $\phi(v_{\ell-1})$ from the transcript, and use \Cref{prop:phi-affine} to compute $\phi(v_\ell)$ as the affine transformation $\phi(\frac{1}{2} v_{\ell-1} + b_{\ell-1})$. After $w$ steps, we thus have $\phi(v_w) = \phi(v / 2^w)$.
    
    Finally, we use an additional $w$ steps of soft token CoT and \Cref{prop:phi-affine} to repeatedly ($w$ times) multiply the result by $2$, resulting in $\phi(v)$ as desired.
\end{proof}

An \emph{alternative} construction for this lemma is to first convert the sequence of bits from least-significant-first to most-significant-first using $w$ steps of attention, and then build $v$ iteratively from $v_{w-1}$ to $v_0$ using the recursion $v_\ell = 2 v_{\ell+1} + b_{\ell+1}$.

\dereferencingLemmaCCoT*

\begin{proof}
Let $S=\{j<\tau : \text{position $j$ is a store-summary}\}$, and for each $j\in S$ let the store-summary’s soft token contain $(\phi(a_j),\phi(v_j))$.
In our restricted simulation, store addresses written at time $j$ satisfy $a_j\le j$, hence (since $j<\tau$) we have $a_j\le \tau$ for all $j\in S$; together with the hypothesis $a\le \tau$, this lets us apply \Cref{lem:phi1-gap} with $T=\tau$.

Let $\delta:=\delta_\tau$ be the corresponding gap, i.e., for every $j\in S$,
\[
\langle \phi(a),\phi(a_j)\rangle =
\begin{cases}
1 & \text{if } a_j=a,\\
\le 1-\delta & \text{if } a_j\neq a.
\end{cases}
\]

\paragraph{Retrieve the most recent matching store (candidate).}
Because store-summary positions are tagged, we can restrict an attention head to attend only over $S$.
We additionally use a fixed-width, strictly increasing \emph{time coordinate} $t_j$ available from the positional embedding
(e.g., $t_j:=-\frac{1}{j+1}\in(-1,0)$), so that $t_j<t_{j'}$ whenever $j<j'$.

Consider a (hard) attention head whose score for a store-summary position $j\in S$ is
\[
\mathrm{score}(j)
\;:=\;
\langle \phi(a),\phi(a_j)\rangle \;+\; \varepsilon\, t_j,
\]
with $\varepsilon:=\delta/4$.
Since $t_j\in(-1,0)$, we have $\varepsilon t_j\in(-\varepsilon,0)$.

If $a_j=a$ then $\mathrm{score}(j)\ge 1-\varepsilon = 1-\delta/4$.
If $a_j\neq a$ then $\mathrm{score}(j)\le (1-\delta)+0 = 1-\delta$.
Thus every matching store-summary strictly outscores every non-matching store-summary.
Among the matching store-summaries, the additive term $\varepsilon t_j$ strictly breaks ties in favor of larger $j$ (because $t_j$ is strictly increasing).
Therefore, if $j^*$ exists, the unique maximizer of $\mathrm{score}(j)$ over $j\in S$ is exactly $j^*=\max\{j\in S : a_j=a\}$, and hard attention retrieves the associated value payload $\phi(v_{j^*})$.
We also route $\phi(a_{j^*})$ into designated residual coordinates alongside $\phi(v_{j^*})$ for a hit-test below; denote the retrieved pair by $(\phi(\tilde a),\phi(\tilde v))$.

If no such $j^*$ exists, the same head returns some store-summary payload $(\phi(\tilde a),\phi(\tilde v))$ with $\tilde a\neq a$.

\paragraph{Hit-test and default to zero.}
To ensure the output is $\phi(0)$ when no match exists, we perform a final hit-test using \Cref{lem:phi1-gap}.
Fix $c:=1-\delta/2$ (note $1-\delta < c < 1$).
Use a second hard-attention head that compares only:
(i) the current position $\tau$, whose key contains $\phi(\tilde a)$ and whose value contains $\phi(\tilde v)$, and
(ii) a fixed anchor position (e.g.\ the start token) whose key contributes the constant score $c$ and whose value is $\phi(0)$.
Concretely, arrange the head so that its score for $\tau$ is $\langle \phi(a),\phi(\tilde a)\rangle$ and its score for the anchor is $c$, with all other positions receiving strictly smaller scores.

By \Cref{lem:phi1-gap}, $\langle \phi(a),\phi(\tilde a)\rangle=1$ iff $\tilde a=a$, and otherwise it is at most $1-\delta<c$.
Hence:
- If a match exists, then $\tilde a=a$ and the head selects the current position, outputting $\phi(\tilde v)=\phi(v_{j^*})$.
- If no match exists, then $\tilde a\neq a$ and the head selects the anchor, outputting $\phi(0)$.

All operations use only constant-dimensional codes and inverse-polynomial margins in $\tau$, so they are implementable with
fixed width and $O(\log\tau)$-bit precision.
\end{proof}

\instructionExecutionLemmaCCoT*

\begin{proof}
    Fix any RAM instruction $I$ as a function from a fixed number of $w$-bit inputs to a fixed number of $w$-bit outputs.
    By assumption, there exists a Turing machine $M_I$ that computes $I$ in time $\instrCostFun{\wordsize}$.
    Using existing Turing-completeness results for transformers \citep{merrill2024cot}, there exists a transformer $T_I$ that simulates $M_I$ (and thus computes $I$) in time $O(\instrCostFun{\wordsize})$.
\end{proof}

\section{Proofs of Hybrid Model Results} \label{app:hybrid}

\serializationLemmaHybrid*

\begin{proof}
    The construction is essentially the same as in the proof of \Cref{lem:ccot-serialization}, except that without continuous CoT, we cannot emit $\phi(v_{k+1})$, defined as $\phi(\lfloor v_k / 2 \rfloor)$, at position $\tau + k + 1$ and read it back in at the following position.
    Instead, we emit $b_k$, computed via \Cref{lem:division}.
    Write $\phi(v)$ as $(s_1, s_2, -s_1, -s_2)$ and observe that $\floor{v_k / 2} = \frac{1}{2} (v_k - b_k)$. We can thus apply \Cref{prop:phi-affine} to compute $\phi(\floor{v_k / 2})$, up to a normalization factor, as:
    \begin{align*}
        \big(\frac{1}{2} s_1, \; s_2, \; -\frac{1}{2} s_1, \; -s_2\big) & & \text{when $b_k = 0$} \\
        \big(\frac{1}{2} (s_1 - s_2), \; s_2, \; -\frac{1}{2} (s_1 - s_2), \; -s_2\big) & & \text{when $b_k = 1$}        
    \end{align*}
    We can use this to design a simple RNN layer that computes $\phi(\floor{v_k / 2})$ at position $\tau + k + 1$ as follows.
    The RNN state $S_t$ will always be diagonal; thus we consider only its vector of diagonal entries $h_t$.
    The hidden state $h_{\tau + k}$ of the RNN is the (unnormalized) $\phi(v_k)$. In particular, $h_{\tau} = \phi(v_0) = \phi(v)$, which can be implemented with a 0 transition matrix and the additive term $\phi(v)$.
    The RNN computes the next hidden state as $h_{\tau + k + 1} = A_k h_{\tau + k}$ where $A_k$ is $B_0$ when $b_k = 0$ and $B_1$ when $b_k = 1$, where
    \begin{small}
    \begin{align*}
        B_0 = \begin{pmatrix}
            1/2 & 0 & 0 & 0 \\
            0 & 1 & 0 & 0 \\
            0 & 0 & 1/2 & 0 \\
            0 & 0 & 0 & 1
        \end{pmatrix}
        , \quad\quad
        B_1 = \begin{pmatrix}
            1/2 & -1/2 & 0 & 0 \\
            0 & 1 & 0 & 0 \\
            0 & 0 & 1/2 & -1/2 \\
            0 & 0 & 0 & 1
        \end{pmatrix}
    \end{align*}
    \end{small}
    
    Observe that $h_{\tau + k + 1}$ is precisely the unnormalized form of $\phi(\floor{v_k / 2})$ derived above. The layer outputs the normalized form $\layernorm(h_{\tau + k + 1})$, which equals $\phi(\floor{v_k / 2})$, to the residual stream, which is then used in the next step to compute $b_{k+1}$ and continue the process.
\end{proof}

While we do not enforce it, the matrices used in the construction have a diagonal-plus-low-rank parameterization: $B_0$ is directly diagonal, and each block of $B_1$ can be decomposed into the identity minus a rank-1 matrix.
Thus, these matrices resemble the parameterizations used in architectures like DeltaNet \citep{pmlr-v139-schlag21a,NEURIPS2024_d13a3eae} and RWKV-7 \citep{peng2025rwkv}.

\deserializationLemmaHybrid*

\begin{proof}
    We use the same recursive computation as in the corresponding lemma for continuous CoT (\Cref{lem:ccot-deserialization}) except within the linear RNN.
    Namely $v_0 = 0$ and, for $1 \leq \ell \leq w$, define $v_\ell = \frac{1}{2} v_{\ell-1} + b_{\ell-1}$. We compute and emit bit $b_{\ell-1}$ at position $\ell$ as before. However, instead of emitting a soft token containing $\phi(v_{\ell-1})$ (which we now cannot do), we simply use $h_\ell = v_\ell$ as the hidden state of the RNN at position $\ell$. This state is then updated as $h_\ell = \frac{1}{2} h_{\ell-1} + b_{\ell-1}$. As before, at position $w$, the hidden state is $h_w = v_w = v / 2^w$. Again as before, we use a simple RNN to multiply this quantity with $2$ an additional $w$ times to obtain $v$, incurring $w$ steps of overhead. Lastly, we add the final output to the residual stream at layer $1$ as $\phi(v)$.
\end{proof}

\section{A Worked Increment Example}
\label{app:increment-example}

We illustrate the polylog-width construction of \Cref{sec:log-width}
with a single increment instruction. Let $\wordsize=4$, let the
program counter be $\terminal{mem}[0]=p$, and suppose instruction
$I_p$ is
\[
  \terminal{mem}[3] \leftarrow \terminal{mem}[3] + 1.
\]
Assume $p+1$ is a valid program address and $p+1<16$.
Here $\bits(x)$ denotes the four bit tokens encoding $x$ in the transcript.

Suppose the transcript contains the following successive writes
to address $3$, with no later write to that address:
\[
\begin{aligned}
  &\cdots\;\memkey\,\BOV\,\bits(3)\,\BOV\,\bits(5)\;\cdots\\
  &\cdots\;\memkey\,\BOV\,\bits(3)\,\BOV\,\bits(7)\;\cdots\;\EOI.
\end{aligned}
\]
These entries record an earlier value of $5$ and the current value
of $7$. The transcript is append-only: updating memory adds a new
entry without removing the old one.

At the final $\EOI$, the simulator reads address $0$ to recover
$p$ and selects instruction $I_p$. It then queries address $3$.
Rightmost matching selects the later write, and the value-retrieval
step recovers $\bits(7)$ rather than $\bits(5)$.
The execution block computes $\bits(8)$ and the next program
counter $\bits(p+1)$.

The store block serializes these two updates, appending
\[
\begin{aligned}
  &\memkey\,\BOV\,\bits(3)\,\BOV\,\bits(8)\\
  &\memkey\,\BOV\,\bits(0)\,\BOV\,\bits(p+1)\,\EOI.
\end{aligned}
\]
The line break is only for readability; each bit is emitted as a
separate token. The output is determined by the result snapshot
computed at the preceding $\EOI$. This snapshot starts the
serialization; at subsequent positions, the store block retrieves
it from that boundary, so partially written updates do not affect
the output.
At the new $\EOI$, the latest entries represent
$\terminal{mem}[3]=8$ and $\terminal{mem}[0]=p+1$.
Thus one RAM instruction is simulated by appending two memory
blocks and an end-of-step marker, using $\bigO(\wordsize)$ tokens.


\newpage
\section*{NeurIPS Paper Checklist}

\begin{enumerate}

\item {\bf Claims}
    \item[] Question: Do the main claims made in the abstract and introduction accurately reflect the paper's contributions and scope?
    \item[] Answer: \answerYes{}
    \item[] Justification: The abstract and introduction state the three main results---polylogarithmic-width transformers (\Cref{sec:log-width}), continuous CoT (\Cref{sec:continuous-cot}), and hybrid transformer--RNN models (\Cref{sec:hybrid})---together with the per-step overheads, and these are formalized in \Cref{thm:logwidthSimulation,thm:wordram-ccot,thm:wordram-hybrid} with proofs.
    \item[] Guidelines:
    \begin{itemize}
        \item The answer \answerNA{} means that the abstract and introduction do not include the claims made in the paper.
        \item The abstract and/or introduction should clearly state the claims made, including the contributions made in the paper and important assumptions and limitations. A \answerNo{} or \answerNA{} answer to this question will not be perceived well by the reviewers. 
        \item The claims made should match theoretical and experimental results, and reflect how much the results can be expected to generalize to other settings. 
        \item It is fine to include aspirational goals as motivation as long as it is clear that these goals are not attained by the paper. 
    \end{itemize}

\item {\bf Limitations}
    \item[] Question: Does the paper discuss the limitations of the work performed by the authors?
    \item[] Answer: \answerYes{} 
    \item[] Justification: \Cref{sec:log-width} and \Cref{sec:fixed-width} discuss the non-uniformity of the polylog-width construction and its reliance on rightmost unique hard attention; the conclusion (\Cref{sec:fixed-width} and the closing section) flags that it remains open whether plain fixed-width log-precision transformers (without continuous CoT or a recurrent layer) can simulate \wRAM{}s with polylog overhead, and notes the $\Omega(w)$ information-theoretic lower bound for discrete CoT.
    \item[] Guidelines:
    \begin{itemize}
        \item The answer \answerNA{} means that the paper has no limitation while the answer \answerNo{} means that the paper has limitations, but those are not discussed in the paper. 
        \item The authors are encouraged to create a separate ``Limitations'' section in their paper.
        \item The paper should point out any strong assumptions and how robust the results are to violations of these assumptions (e.g., independence assumptions, noiseless settings, model well-specification, asymptotic approximations only holding locally). The authors should reflect on how these assumptions might be violated in practice and what the implications would be.
        \item The authors should reflect on the scope of the claims made, e.g., if the approach was only tested on a few datasets or with a few runs. In general, empirical results often depend on implicit assumptions, which should be articulated.
        \item The authors should reflect on the factors that influence the performance of the approach. For example, a facial recognition algorithm may perform poorly when image resolution is low or images are taken in low lighting. Or a speech-to-text system might not be used reliably to provide closed captions for online lectures because it fails to handle technical jargon.
        \item The authors should discuss the computational efficiency of the proposed algorithms and how they scale with dataset size.
        \item If applicable, the authors should discuss possible limitations of their approach to address problems of privacy and fairness.
        \item While the authors might fear that complete honesty about limitations might be used by reviewers as grounds for rejection, a worse outcome might be that reviewers discover limitations that aren't acknowledged in the paper. The authors should use their best judgment and recognize that individual actions in favor of transparency play an important role in developing norms that preserve the integrity of the community. Reviewers will be specifically instructed to not penalize honesty concerning limitations.
    \end{itemize}

\item {\bf Theory assumptions and proofs}
    \item[] Question: For each theoretical result, does the paper provide the full set of assumptions and a complete (and correct) proof?
    \item[] Answer: \answerYes{} 
    \item[] Justification: All theorems, lemmas, and propositions are numbered and cross-referenced. Assumptions on the transformer model (precision, width, attention type, uniformity) are stated in \Cref{subsec:transformers} and restated in each theorem. Proof sketches appear in the main text and full proofs are provided in the appendix.
    \item[] Guidelines:
    \begin{itemize}
        \item The answer \answerNA{} means that the paper does not include theoretical results. 
        \item All the theorems, formulas, and proofs in the paper should be numbered and cross-referenced.
        \item All assumptions should be clearly stated or referenced in the statement of any theorems.
        \item The proofs can either appear in the main paper or the supplemental material, but if they appear in the supplemental material, the authors are encouraged to provide a short proof sketch to provide intuition. 
        \item Inversely, any informal proof provided in the core of the paper should be complemented by formal proofs provided in appendix or supplemental material.
        \item Theorems and Lemmas that the proof relies upon should be properly referenced. 
    \end{itemize}

    \item {\bf Experimental result reproducibility}
    \item[] Question: Does the paper fully disclose all the information needed to reproduce the main experimental results of the paper to the extent that it affects the main claims and/or conclusions of the paper (regardless of whether the code and data are provided or not)?
    \item[] Answer: \answerNA{} 
    \item[] Justification: The paper is purely theoretical and contains no experiments.
    \item[] Guidelines:
    \begin{itemize}
        \item The answer \answerNA{} means that the paper does not include experiments.
        \item If the paper includes experiments, a \answerNo{} answer to this question will not be perceived well by the reviewers: Making the paper reproducible is important, regardless of whether the code and data are provided or not.
        \item If the contribution is a dataset and\slash or model, the authors should describe the steps taken to make their results reproducible or verifiable. 
        \item Depending on the contribution, reproducibility can be accomplished in various ways. For example, if the contribution is a novel architecture, describing the architecture fully might suffice, or if the contribution is a specific model and empirical evaluation, it may be necessary to either make it possible for others to replicate the model with the same dataset, or provide access to the model. In general. releasing code and data is often one good way to accomplish this, but reproducibility can also be provided via detailed instructions for how to replicate the results, access to a hosted model (e.g., in the case of a large language model), releasing of a model checkpoint, or other means that are appropriate to the research performed.
        \item While NeurIPS does not require releasing code, the conference does require all submissions to provide some reasonable avenue for reproducibility, which may depend on the nature of the contribution. For example
        \begin{enumerate}
            \item If the contribution is primarily a new algorithm, the paper should make it clear how to reproduce that algorithm.
            \item If the contribution is primarily a new model architecture, the paper should describe the architecture clearly and fully.
            \item If the contribution is a new model (e.g., a large language model), then there should either be a way to access this model for reproducing the results or a way to reproduce the model (e.g., with an open-source dataset or instructions for how to construct the dataset).
            \item We recognize that reproducibility may be tricky in some cases, in which case authors are welcome to describe the particular way they provide for reproducibility. In the case of closed-source models, it may be that access to the model is limited in some way (e.g., to registered users), but it should be possible for other researchers to have some path to reproducing or verifying the results.
        \end{enumerate}
    \end{itemize}

\item {\bf Open access to data and code}
    \item[] Question: Does the paper provide open access to the data and code, with sufficient instructions to faithfully reproduce the main experimental results, as described in supplemental material?
    \item[] Answer: \answerNA{} 
    \item[] Justification: No code, data, or model artifacts are associated with the paper; all results are mathematical proofs.
    \item[] Guidelines:
    \begin{itemize}
        \item The answer \answerNA{} means that paper does not include experiments requiring code.
        \item Please see the NeurIPS code and data submission guidelines (\url{https://neurips.cc/public/guides/CodeSubmissionPolicy}) for more details.
        \item While we encourage the release of code and data, we understand that this might not be possible, so \answerNo{} is an acceptable answer. Papers cannot be rejected simply for not including code, unless this is central to the contribution (e.g., for a new open-source benchmark).
        \item The instructions should contain the exact command and environment needed to run to reproduce the results. See the NeurIPS code and data submission guidelines (\url{https://neurips.cc/public/guides/CodeSubmissionPolicy}) for more details.
        \item The authors should provide instructions on data access and preparation, including how to access the raw data, preprocessed data, intermediate data, and generated data, etc.
        \item The authors should provide scripts to reproduce all experimental results for the new proposed method and baselines. If only a subset of experiments are reproducible, they should state which ones are omitted from the script and why.
        \item At submission time, to preserve anonymity, the authors should release anonymized versions (if applicable).
        \item Providing as much information as possible in supplemental material (appended to the paper) is recommended, but including URLs to data and code is permitted.
    \end{itemize}

\item {\bf Experimental setting/details}
    \item[] Question: Does the paper specify all the training and test details (e.g., data splits, hyperparameters, how they were chosen, type of optimizer) necessary to understand the results?
    \item[] Answer: \answerNA{} 
    \item[] Justification: The paper contains no experiments.
    \item[] Guidelines:
    \begin{itemize}
        \item The answer \answerNA{} means that the paper does not include experiments.
        \item The experimental setting should be presented in the core of the paper to a level of detail that is necessary to appreciate the results and make sense of them.
        \item The full details can be provided either with the code, in appendix, or as supplemental material.
    \end{itemize}

\item {\bf Experiment statistical significance}
    \item[] Question: Does the paper report error bars suitably and correctly defined or other appropriate information about the statistical significance of the experiments?
    \item[] Answer: \answerNA{} 
    \item[] Justification: The paper contains no experiments.
    \item[] Guidelines:
    \begin{itemize}
        \item The answer \answerNA{} means that the paper does not include experiments.
        \item The authors should answer \answerYes{} if the results are accompanied by error bars, confidence intervals, or statistical significance tests, at least for the experiments that support the main claims of the paper.
        \item The factors of variability that the error bars are capturing should be clearly stated (for example, train/test split, initialization, random drawing of some parameter, or overall run with given experimental conditions).
        \item The method for calculating the error bars should be explained (closed form formula, call to a library function, bootstrap, etc.)
        \item The assumptions made should be given (e.g., Normally distributed errors).
        \item It should be clear whether the error bar is the standard deviation or the standard error of the mean.
        \item It is OK to report 1-sigma error bars, but one should state it. The authors should preferably report a 2-sigma error bar than state that they have a 96\% CI, if the hypothesis of Normality of errors is not verified.
        \item For asymmetric distributions, the authors should be careful not to show in tables or figures symmetric error bars that would yield results that are out of range (e.g., negative error rates).
        \item If error bars are reported in tables or plots, the authors should explain in the text how they were calculated and reference the corresponding figures or tables in the text.
    \end{itemize}

\item {\bf Experiments compute resources}
    \item[] Question: For each experiment, does the paper provide sufficient information on the computer resources (type of compute workers, memory, time of execution) needed to reproduce the experiments?
    \item[] Answer: \answerNA{} 
    \item[] Justification:  The paper contains no experiments.
    \item[] Guidelines:
    \begin{itemize}
        \item The answer \answerNA{} means that the paper does not include experiments.
        \item The paper should indicate the type of compute workers CPU or GPU, internal cluster, or cloud provider, including relevant memory and storage.
        \item The paper should provide the amount of compute required for each of the individual experimental runs as well as estimate the total compute. 
        \item The paper should disclose whether the full research project required more compute than the experiments reported in the paper (e.g., preliminary or failed experiments that didn't make it into the paper). 
    \end{itemize}
    
\item {\bf Code of ethics}
    \item[] Question: Does the research conducted in the paper conform, in every respect, with the NeurIPS Code of Ethics \url{https://neurips.cc/public/EthicsGuidelines}?
    \item[] Answer: \answerYes{} 
    \item[] Justification: The authors have reviewed the NeurIPS Code of Ethics. The work is foundational theory on the expressive power of chain-of-thought transformers and raises no ethical concerns related to data, human subjects, or deployment.
    \item[] Guidelines:
    \begin{itemize}
        \item The answer \answerNA{} means that the authors have not reviewed the NeurIPS Code of Ethics.
        \item If the authors answer \answerNo, they should explain the special circumstances that require a deviation from the Code of Ethics.
        \item The authors should make sure to preserve anonymity (e.g., if there is a special consideration due to laws or regulations in their jurisdiction).
    \end{itemize}

\item {\bf Broader impacts}
    \item[] Question: Does the paper discuss both potential positive societal impacts and negative societal impacts of the work performed?
    \item[] Answer: \answerYes{} 
    \item[] Justification: A clearer theoretical understanding of chain-of-thought expressivity can inform the design, interpretability, and analysis of reasoning models, which we view as broadly positive. The paper introduces no new models, datasets, or deployable systems, and does not enable any specific capability beyond clarifying what current architectures can in principle simulate.
    \item[] Guidelines:
    \begin{itemize}
        \item The answer \answerNA{} means that there is no societal impact of the work performed.
        \item If the authors answer \answerNA{} or \answerNo, they should explain why their work has no societal impact or why the paper does not address societal impact.
        \item Examples of negative societal impacts include potential malicious or unintended uses (e.g., disinformation, generating fake profiles, surveillance), fairness considerations (e.g., deployment of technologies that could make decisions that unfairly impact specific groups), privacy considerations, and security considerations.
        \item The conference expects that many papers will be foundational research and not tied to particular applications, let alone deployments. However, if there is a direct path to any negative applications, the authors should point it out. For example, it is legitimate to point out that an improvement in the quality of generative models could be used to generate Deepfakes for disinformation. On the other hand, it is not needed to point out that a generic algorithm for optimizing neural networks could enable people to train models that generate Deepfakes faster.
        \item The authors should consider possible harms that could arise when the technology is being used as intended and functioning correctly, harms that could arise when the technology is being used as intended but gives incorrect results, and harms following from (intentional or unintentional) misuse of the technology.
        \item If there are negative societal impacts, the authors could also discuss possible mitigation strategies (e.g., gated release of models, providing defenses in addition to attacks, mechanisms for monitoring misuse, mechanisms to monitor how a system learns from feedback over time, improving the efficiency and accessibility of ML).
    \end{itemize}
    
\item {\bf Safeguards}
    \item[] Question: Does the paper describe safeguards that have been put in place for responsible release of data or models that have a high risk for misuse (e.g., pre-trained language models, image generators, or scraped datasets)?
    \item[] Answer: \answerNA{} 
    \item[] Justification: The paper does not release any models, datasets, or other artifacts that pose a risk of misuse.
    \item[] Guidelines:
    \begin{itemize}
        \item The answer \answerNA{} means that the paper poses no such risks.
        \item Released models that have a high risk for misuse or dual-use should be released with necessary safeguards to allow for controlled use of the model, for example by requiring that users adhere to usage guidelines or restrictions to access the model or implementing safety filters. 
        \item Datasets that have been scraped from the Internet could pose safety risks. The authors should describe how they avoided releasing unsafe images.
        \item We recognize that providing effective safeguards is challenging, and many papers do not require this, but we encourage authors to take this into account and make a best faith effort.
    \end{itemize}

\item {\bf Licenses for existing assets}
    \item[] Question: Are the creators or original owners of assets (e.g., code, data, models), used in the paper, properly credited and are the license and terms of use explicitly mentioned and properly respected?
    \item[] Answer: \answerNA{} 
    \item[] Justification: The paper does not use existing code, datasets, or model assets; all prior work is credited via citations in the bibliography.
    \item[] Guidelines:
    \begin{itemize}
        \item The answer \answerNA{} means that the paper does not use existing assets.
        \item The authors should cite the original paper that produced the code package or dataset.
        \item The authors should state which version of the asset is used and, if possible, include a URL.
        \item The name of the license (e.g., CC-BY 4.0) should be included for each asset.
        \item For scraped data from a particular source (e.g., website), the copyright and terms of service of that source should be provided.
        \item If assets are released, the license, copyright information, and terms of use in the package should be provided. For popular datasets, \url{paperswithcode.com/datasets} has curated licenses for some datasets. Their licensing guide can help determine the license of a dataset.
        \item For existing datasets that are re-packaged, both the original license and the license of the derived asset (if it has changed) should be provided.
        \item If this information is not available online, the authors are encouraged to reach out to the asset's creators.
    \end{itemize}

\item {\bf New assets}
    \item[] Question: Are new assets introduced in the paper well documented and is the documentation provided alongside the assets?
    \item[] Answer: \answerNA{} 
    \item[] Justification: No new datasets, code, or model artifacts are released with the paper.
    \item[] Guidelines:
    \begin{itemize}
        \item The answer \answerNA{} means that the paper does not release new assets.
        \item Researchers should communicate the details of the dataset\slash code\slash model as part of their submissions via structured templates. This includes details about training, license, limitations, etc. 
        \item The paper should discuss whether and how consent was obtained from people whose asset is used.
        \item At submission time, remember to anonymize your assets (if applicable). You can either create an anonymized URL or include an anonymized zip file.
    \end{itemize}

\item {\bf Crowdsourcing and research with human subjects}
    \item[] Question: For crowdsourcing experiments and research with human subjects, does the paper include the full text of instructions given to participants and screenshots, if applicable, as well as details about compensation (if any)? 
    \item[] Answer: \answerNA{} 
    \item[] Justification: The paper does not involve crowdsourcing or research with human subjects.
    \item[] Guidelines:
    \begin{itemize}
        \item The answer \answerNA{} means that the paper does not involve crowdsourcing nor research with human subjects.
        \item Including this information in the supplemental material is fine, but if the main contribution of the paper involves human subjects, then as much detail as possible should be included in the main paper. 
        \item According to the NeurIPS Code of Ethics, workers involved in data collection, curation, or other labor should be paid at least the minimum wage in the country of the data collector. 
    \end{itemize}

\item {\bf Institutional review board (IRB) approvals or equivalent for research with human subjects}
    \item[] Question: Does the paper describe potential risks incurred by study participants, whether such risks were disclosed to the subjects, and whether Institutional Review Board (IRB) approvals (or an equivalent approval/review based on the requirements of your country or institution) were obtained?
    \item[] Answer: \answerNA{} 
    \item[] Justification: The paper does not involve research with human subjects.
    \item[] Guidelines:
    \begin{itemize}
        \item The answer \answerNA{} means that the paper does not involve crowdsourcing nor research with human subjects.
        \item Depending on the country in which research is conducted, IRB approval (or equivalent) may be required for any human subjects research. If you obtained IRB approval, you should clearly state this in the paper. 
        \item We recognize that the procedures for this may vary significantly between institutions and locations, and we expect authors to adhere to the NeurIPS Code of Ethics and the guidelines for their institution. 
        \item For initial submissions, do not include any information that would break anonymity (if applicable), such as the institution conducting the review.
    \end{itemize}

\item {\bf Declaration of LLM usage}
    \item[] Question: Does the paper describe the usage of LLMs if it is an important, original, or non-standard component of the core methods in this research? Note that if the LLM is used only for writing, editing, or formatting purposes and does \emph{not} impact the core methodology, scientific rigor, or originality of the research, declaration is not required.
    \item[] Answer: \answerNA{} 
    \item[] Justification: Transformers are the formal object of study in this paper, not a methodological tool; LLMs were not used as an original or non-standard component of the core methods.
    \item[] Guidelines:
    \begin{itemize}
        \item The answer \answerNA{} means that the core method development in this research does not involve LLMs as any important, original, or non-standard components.
        \item Please refer to our LLM policy in the NeurIPS handbook for what should or should not be described.
    \end{itemize}

\end{enumerate}

\end{document}